\documentclass[12pt]{article}
\usepackage{amsmath}
\usepackage{graphicx}
\usepackage{enumerate}
\usepackage[authoryear,round]{natbib}
\usepackage{url} 
\usepackage{xcolor}
\usepackage{amsthm,amsfonts,amssymb}
\usepackage{bm}
\usepackage{booktabs}
\usepackage{multirow}
\usepackage[textsize=scriptsize]{todonotes}
\usepackage{setspace}

\def\m0{\mathbf{0}}
\def\mA{\mathbf{A}}
\def\mB{\mathbf{B}}
\def\vb{\mathbf{b}}
\def\Dbb{\mathbb{D}}

\def\Ebb{\mathbb{E}}
\def\cF{\mathcal{F}}

\def\cH{\mathcal{H}}
\def\bI{\mathbf{I}}
\def\mI{\mathbf{I}}
\def\cL{\mathcal{L}}
\def\cM{\mathcal{M}}
\def\cN{\mathcal{N}}
\def\mQ{\mathbf{Q}}
\def\real{\mathbb{R}}

\def\vs{\mathbf{s}}
\def\cS{\mathcal{S}}
\def\vt{\mathbf{t}}
\def\cV{\mathcal{V}}
\def\cW{\mathcal{W}}

\def\rx{\mathbf{x}}
\def\vx{\mathbf{x}}
\def\vy{\mathbf{y}}
\def\vz{\mathbf{z}}

\def\real{\mathbb{R}}
\def\vtheta{\boldsymbol{\theta}}
\def\vphi{\boldsymbol{\phi}}
\def\vveps{\boldsymbol{\varepsilon}}
\newcommand{\indep}{\;\, \rule[0em]{.03em}{.67em} \hspace{-.25em}
	\rule[0em]{.65em}{.03em} \hspace{-.25em}
	\rule[0em]{.03em}{.67em}\;\,}

\def\Cov{\text{Cov}}
\def\argmin{\text{argmin}}

\def\rvx{{\mathbf{x}}}
\def\rvy{{\mathbf{y}}}
\def\rvz{{\mathbf{z}}}

\def\vveps{{\bm{\varepsilon}}}
\def\vtheta{{\bm{\theta}}}
\def\vphi{{\bm{\phi}}}

\def\vb{{\bm{b}}}

\def\vs{{\bm{s}}}
\def\vt{{\bm{t}}}

\def\vx{{\bm{x}}}
\def\vy{{\bm{y}}}
\def\vz{{\bm{z}}}

\newtheorem{theorem}{Theorem}[section]
\newtheorem{lemma}[theorem]{Lemma}

\newcommand{\blind}{1}

\begin{document}

\def\spacingset#1{\renewcommand{\baselinestretch}%
{#1}\small\normalsize} \spacingset{1}


\if1\blind
{
  \title{\bf Deep Dimension Reduction for Supervised Representation Learning}
  \author{Jian Huang\thanks{
    Supported in part by the NSF grant DMS-1916199}
    \\
    Department of Statistics and Actuarial Science, University of Iowa, USA \\
    Yuling Jiao\thanks{Supported in part by the NSFC
grants No.11871474 and No.61701547}
    \\
    School of Mathematics and Statistics, Wuhan University, China \\
    Xu Liao \\
    Center of Quantitative Medicine,
  Duke-NUS Medical School, Singapore\\
  Jin Liu\thanks{Supported in part by Duke-NUS  Medical School WBS R-913-200-098-263 and Singapore MOE grants 2016-T2-2-029, 2018-T2-2-006 and 2018-T2-1-046}\\
  Center of Quantitative Medicine,
  Duke-NUS Medical School, Singapore\\
  and \\
  Zhou Yu \\
  School of Statistics,
  East China Normal University,
  China }
  \maketitle
} \fi

\if0\blind
{
  \bigskip
  \bigskip
  \bigskip
  \begin{center}
    {\LARGE\bf Deep Dimension Reduction for Supervised Representation Learning}
\end{center}
  \medskip
} \fi

\bigskip
\begin{abstract}
The goal of supervised representation learning is to construct effective data representations for prediction. Among all the characteristics of an ideal nonparametric representation of high-dimensional complex data, sufficiency,  low dimensionality and disentanglement are some of the most essential ones. We propose a deep dimension reduction approach to learning representations with these characteristics. The proposed approach is a nonparametric generalization of the sufficient dimension reduction method. We formulate the ideal representation learning task as that of finding a nonparametric representation that minimizes an objective function characterizing conditional independence and promoting disentanglement at the population level. We then estimate the target representation at the sample level nonparametrically using deep neural networks. We show that the estimated deep nonparametric representation is consistent in the sense that its excess risk converges to zero. Our extensive numerical experiments using simulated  and real benchmark data demonstrate that the proposed methods have better performance than several existing dimension reduction methods and the standard deep learning models in the context of classification and regression.
\end{abstract}

\noindent%
\vfill

\newpage
\spacingset{1.5} 
\section{Introduction}
\label{sec:intro}

Over the past decade, deep learning has achieved impressive successes in modeling high-dimensional complex data
 arising from many scientific fields.
A key factor for these successes is the ability of certain neural network models to learn nonlinear representations from complex high-dimensional data \citep{bengio2013representation, lecun2015deep}. For example, convolutional neural networks
are able to learn effective representations of image data \citep{LeCunBoserDenkerEtAl989}.
However, in general, optimizing the standard cross-entropy loss for classification and the least squares loss for regression do not guarantee that the learned representations enjoy any desired properties \citep{alain2017understanding}. Therefore, it is imperative to develop principled approaches for constructing effective data representations.
Representation learning has emerged as an important framework for modeling complex data
\citep{bengio2013representation}, with wide applications in
classification, regression, imaging analysis, 
domain adaptation and transfer learning, among others. The goal of supervised representation learning is to construct effective  representations of high-dimensional input data for various supervised learning tasks.
In this paper, we propose a deep dimension reduction (DDR) method for sufficient representation learning. DDR aims at estimating a sufficient representation nonparametrically using deep neural networks based on the conditional independence principle.

There is a large body of literature on dimension reduction  in statistics and machine learning. A prominent approach for supervised dimension reduction and representation learning is the sufficient dimension reduction (SDR) introduced in the seminal paper by \cite{li1991sliced}.
A key aspect that distinguishes SDR from many other dimension reduction methods is that it does not make any model assumptions on the conditional distribution of the response given the predictors. In the framework of SDR,  a semiparametric method, called sliced inverse regression (SIR), was first proposed for estimating the linear dimension reduction direction, or linear sufficient representation \citep{li1991sliced}. The SIR and related methods were further developed by many researchers,  see, for example,   \citet{cook1991sliced}, \citet{ li1992principal}, \citet{yin2002}, \citet{cook1998regression}, 
 \citet{li2005contour} and \citet{zhu2010cumu} and the references therein.
These methods require the linearity and constant covariance conditions on the distribution of the predictors. Several approaches have been developed without assuming these conditions, including methods based on nonparametric regression \citep{xia2002adaptive},  conditional covariance operators \citep{
fukumizu2009kernel},  mutual information \citep{suzuki2013sufficient}, distance correlation
\citep{vepakomma2018supervised}, and semiparametric modeling \citep{ma2012semi, ma2013effi}.
These SDR methods focus on linear dimension reduction, that is, the features learned are linear functions of the original input variables.  However, linear functions may not be adequate for representing high-dimensional complex data such as images and natural languages,  due to the highly nonlinear nature of such data. \citet{lee2013gen} formulated a general sufficient dimension reduction framework in the nonlinear setting and proposed a generalized inverse regression approach using conditional covariance operators,  but this method is computationally prohibitive with high-dimensional data such as the image datasets considered in Section \ref{experiments}.
We refer  to  the review papers \citep{cook2007fisher, cook2018review, ma2013review} and  the monograph \citep{li2018sufficient} for thorough reviews of SDR methods.

Among all the characteristics of an ideal representation for supervised learning, sufficiency, low dimensionality and disentanglement are some of the most essential ones \citep{achille2018emergence}.
Sufficiency is a basic property a representation should have.
It  is closely related to the concept of sufficient statistics in a parametric model  \citep{Fisher1922MathFound,cook2007fisher}. In supervised representation learning, sufficiency is characterized by the conditional independence principle, which states that the original input data is conditionally independent of the response given the representation. In other words,  a sufficient representation contains all the relevant information in the input data about the response.
Low dimensionality means that the representation should have as few components as possible to represent the underlying structure of the data, and the number of components should be fewer than the ambient dimension.
In the context of nonparametric representation learning, disentanglement refers to the requirement that the components of the representation should be statistically independent. This is an extension of and stronger than the orthogonal constraint in the linear representation setting, where the components of
 the linear representation are constrained to have orthonormal directions.
The notion of disentanglement is based on the hypothesis that there are some underlying factors determining the data generation process: although the observed data are high-dimensional and complex, the underlying factors are low-dimensional, disentangled,
and have a simple statistical structure.
The components in the learned representation can often be interpreted as corresponding to the latent structure of the observed data,
thus disentanglement is an important property for better separating latent factors from one to another. 
A representation with these characteristics can make the model more interpretable and facilitates the downstream supervised learning tasks.

Inspired by the basic idea of SDR, we propose a deep dimension reduction (DDR) approach for supervised representation learning with the properties of  sufficiency, low dimensionality and disentanglement.
By taking the advantage of the strong capacities of deep neural networks in approximating high-dimensional functions for nonparametric estimation, we model the DDR representations, which we refer to as DDR map (DDRM) for convenience,  using deep neural networks to capture the nonlinearity in the representation space. It would be difficult to use the traditional techniques for nonparametric estimation such as kernel smoothing and splines for multi- or high-dimensional function estimation in the context of representation learning.
To characterize the conditional independence of the representation, we use the distance covariance \citep{szekely2007measuring} as the conditional independence measure that can be computed efficiently. We also promote the disentanglement for  DDRM by
regularizing its distribution to have independent components based on a divergence measure.

Our main contributions are as follows:
\begin{itemize}
 \setlength\itemsep{-0.05 cm}
\item
We formulate a new nonparametric approach to  dimension reduction by characterizing the  sufficient dimension reduction map as a minimizer of a loss function measuring conditional independence and disentanglement.
\item
We estimate the sufficient dimension reduction map at the sample level nonparametrically using deep neural networks based on distance covariance for characterizing sufficiency and use $f$-divergence to promote disentanglement of the learned representation.


\item We show that the estimated deep dimension reduction map is consistent in the sense that it achieves
asymptotic sufficiency under mild conditions.

\item We validate DDR via comprehensive numerical experiments and real data analysis in the
context of regression and classification.
We use the learned features based on DDR as inputs for linear regression and nearest neighbor classification. The resulting prediction accuracies are better than  those based on linear dimension reduction methods  for regression and   deep learning models for classification. The PyTorch code for DDR is available at {\url{https://github.com/anonymous/DDR}}.
\end{itemize}

The rest of the paper is organized as follows. In Section \ref{setup} we discuss  the theoretical framework for learning a DDRM.
This framework leads to the formulation of an objective function using distance correlation for characterizing conditional independence in Section \ref{ddrm}.
We estimate the target DDRM based on the sample version of the objective function using deep neural networks and develop an efficient algorithm for training the DDRM.
In Section \ref{theory} we provide sufficient conditions under which estimated nonparametric representations achieves asymptotic sufficiency.
This result provides strong theoretical support for the proposed method.
The algorithm for implementing DDR is described in Section \ref{algorithm}.
In Section \ref{experiments} we validate the proposed DDR via extensive numerical experiments and real data examples.

\section{Sufficient representation and distance correlation}
\label{setup}
Consider a pair of random vectors $(X,Y) \in \real^p \times \real^q$, where $X$ is a vector of predictors and $Y$ is a vector of response variables or labels.
Our goal is to construct a representation of $X$ that possesses the three characteristics: sufficiency, low dimensionality and disentanglement.

\subsection{Sufficiency}
A measurable function $\vs : \mathbb{R}^{p} \rightarrow \mathbb{R}^{d}$ with $d \le p$ is said to be a sufficient representation of $X$ if
\begin{equation}\label{cida}
Y \indep X | \vs(X),
\end{equation}
that is, $Y$ and $X$ are conditionally independent given $\vs(X)$.
This condition holds if and only if the conditional distribution of $Y$ given $X$ and that of $Y$ given $\vs(X)$ are equal.
Therefore, the information in $X$ about $Y$ is completely encoded by $\vs(X).$
Such a function $\vs$ always exists, since if we simply take $\vs(\vx)=\vx$, then (\ref{cida}) holds trivially.
This formulation is a nonparametric generalization of the basic condition in sufficient dimension reduction \citep{li1991sliced, cook1998regression}, where it is assumed $\vs(\vx)=\mB^T \vx$ with
$\mB\in \real^{p \times d}$ belonging to the Stiefel manifold, i.e., $\mB^T\mB=\mI_{d}.$

Denote the class of sufficient representations satisfying (\ref{cida}) by
\[\cF = \{\vs: \real^p \to \real^{d}, \vs  \text{
  satisfies } Y\indep X| \vs(X) \}.
\]
For an injective measurable transformation $T: \real^{d} \to \real^{d}$ and $\vs \in \cF$, $T\circ \vs(X)$ is also sufficient by the basic property of conditional probability. Therefore,
the class $\cF$ is invariant in the sense that
\[
T\circ \cF \subseteq \cF, \ \text{ provided } T \text{ is injective,}
\]
where $T\circ \cF = \{T\circ \vs: \vs \in \cF\}.$
An important class of transformations is the class of affine transformations, $T\circ \vs = \mA \vs + \vb$, where $\mA$ is a ${d\times d}$ nonsingular matrix and $\vb \in \real^{d}.$


\subsection{Space of nonparametric sufficient representations}
%
%

The  nonparametric sufficient representations are nonunique and the space of such representations is large, since if (\ref{cida}) holds for $\vs$, it also holds for any one-to-one transformation of $\vs$. We propose to narrow the space of such representations by constraining the distributional properties of $\vs(\vx)$.

Among the sufficient representations, it is preferable to have those with a simple statistical distribution and whose components are independent, that is, the components are disentangled.
For a sufficient representation $\vs(X)$, let $\Sigma_{\vs}=\Cov(\vs(X))$.  Suppose $\Sigma_{\vs}$ is positive definite,  then $\Sigma_{\vs}^{-1/2}\vs(X)$ is also a sufficient representation.  Therefore, we can always rescale $\vs(X)$ such that it has identity covariance matrix. To further simplify the statistical structure of a representation $\vs$, we also impose the constraint that it is rotation invariant in distribution, that is, $\mQ\vs(X)=\vs(X)$ in distribution for any orthogonal matrix $\mQ \in \real^{d\times d}$.
By the Maxwell characterization of the Gaussian distributions
\citep{maxwell1860, bryc1995normal},
a random vector of dimension two or more with independent components is rotation invariant in distribution if and only if it is Gaussian with zero mean and a spherical covariance matrix. Therefore, after absorbing the scaling factor, for a sufficient representation map to be have independent components and be rotation invariant, it is necessarily distributed as
 $N_{d}(\m0, \mI_{d})$.
Denote
\begin{equation}
\label{maxwell}
\cM=\{R: \real^p \to \real^d, R(X) \sim \mathcal{N}(\m0, \mI_{d})\}.
\end{equation}
Now our problem becomes that of finding a representation in
 $\cF\cap \cM$, the intersection of the Fisher class and the Maxwell class.

Does such a sufficient representation exist?
The following result from the optimal transport theory gives an affirmative answer and guarantees the existence of such a representation under mild conditions
\citep{
villani2008optimal}.

\begin{lemma}\label {lem1}
	Let $\mu$ be a probability measure on $\mathbb{R}^{d}$.
Suppose it has finite second 
moment and is absolutely continuous with respect to the standard Gaussian measure, denoted by $\gamma_{d}$.  Then it admits a unique optimal transportation map ${T} : \mathbb{R}^{d} \rightarrow \mathbb{R}^{d}$
	such that ${T}_{\#} \mu = \gamma_{d}\equiv \mathcal{N}(\m0,\mI_{d})$, where ${T}_{\#} \mu$ denotes the pushforward distribution of $\mu$ under ${T}$. Moreover,  ${T}$ is  injective $\mu$-almost everywhere.
\end{lemma}

Denote the law of a random vector $Z$ by $\mu_{Z}$.
Lemma \ref{lem1} implies that, for any $\vs \in \cF$ with $\Ebb \|\vs(X)\|^2 < \infty$ and $\mu_{\vs(X)}$ absolutely continuous with respect to $\gamma_{d}$, there exists a map  ${T}^*$ transforming the distribution of  $\vs(X)$
to $\mathcal{N}(\m0, \mI_d)$.
Therefore,
$R^* := {T}^* \circ \vs \in \cF \cap \cM$, that is,
\begin{equation}\label{cida2}
X \indep Y | R^*(X) \ \text{ and } \  {R}^*(X) \sim \mathcal{N}(\m0,\mI_{d}).
\end{equation}
The requirement that $R^*(X) \sim \cN(\m0, \mI_d)$ can be considered a regularization on the distribution of $R^*(X)$.  This is similar to the ridge regression where the ridge penalty can be derived from  a spherical normal prior on the regression coefficient.
We use the standard multivariate normal distribution as the reference for regularizing the distribution of the sufficient representation. It is possible to use other distributions such as uniform distribution on the unit cube $[0,1]^{d}$.
Below, to be specific, we will focus on using the standard normal distribution for regularization. Also, we note that it suffices to estimate the function $R^*$, not $\vs$ and $T^*$ separately, since $R^*$ satisfies the conditional independence requirement.

{\color{black}
The independence requirement for the components of $R^*$ is reminiscent of the
same requirement in the independent component analysis (ICA, \citet{jutten1991, comon1994}).
ICA is a method for estimating hidden factors that underlie a random vector $X$.
It posits that $X$ is a \textit{linear transformation} of an unknown random vector with independent components, but the transformation is unknown. The goal of ICA is to estimate this linear transformation.
DDR differs from ICA in three crucial aspects. First, DDR is a supervised method
that seeks to find a data representation such that the response is conditionally independent given this representation, while ICA is an unsupervised method that attempts to identify independent latent factors
underlying 
the original data vector. Second, DDR seeks a nonparametric function $R^*$ such that $R^*(X)$ has independent components, while ICA attempts to find a matrix $W \in \mathbb{R}^{p \times p}$ such that the components of $WX$ are independent. Third, the distribution of $R^*(X)$ can be Gaussian; in contrast, a basic restriction in ICA is that the independent components must be non-Gaussian. There is a large body of literature on ICA.
For some more recent references on ICA, see \citet{samarov2004}, \citet{samworth2012} and the review
 \citet{nordhausen2018}.
Feedforward neural networks and recurrent neural
network structures have also been considered in solving ICA problems \citep{mutihac2003}.
We  refer the reader to the monographs \citep{ica2001a, ica2001b} for additional references on ICA.}


\section{Nonparametric estimation of representation map}
\label{ddrm}
The discussions in Section \ref{setup} lay the ground for formulating an objective function that can be used for constructing a DDRM $R^*$ satisfying (\ref{cida2}), that is, $R^*$ is sufficient and disentangled.

\subsection{Population objective function}
Let $\cV$ be a measure of dependence between random variables $X$ and $Y$ with the following properties:
(a) $\cV[X,Y] \ge 0$ with $\cV[X, Y]=0$ if and only if $X\indep Y$;
(b) $\cV[X, Y] \ge \cV[R(X), Y]$ for all measurable function $R$;
and (c) $\cV[X,Y]=\cV[R^*(X), Y] \ \text{ if and only if } R^*\in \cF.$
These properties imply that
$
R^* \in \cF $ if and only of $ R^* \in
\argmin_{R}\{ -\cV[R(X), Y]\}.
$

For the normality regularization in (\ref{cida2}),
we use a divergence measure $\mathbb{D}$ 
to quantify the difference between
$\mu_{R(X)}$ and the standard  normal distribution $\gamma_{d}$.
This measure should satisfy the condition
$\Dbb(\mu_{R(X)}\Vert \gamma_{d})\ge 0 $ for every measurable function $R$
\text{ and } $\Dbb(\mu_{R(X)}\Vert \gamma_{d})=0$  \text{ if and only if }
$ \ R \in \cM. $
The $f$-divergences, including the KL-divergence, satisfy this condition.
It follows that
$R^* \in \cM$  if and only if $ R^* \in \argmin_{R} \Dbb(\mu_{R(X)}\| \gamma_{d}).$
Then the problem of finding a sufficient and disentangle map
$R^*$
becomes a constrained minimization problem:
\[
\argmin_{R} -\cV[R(X), Y]\  \text{ subject to } \ \Dbb(\mu_{R(X)}\Vert\gamma_{d})=0.
\]
The Lagrangian form of this minimization problem is
\begin{equation}
\label{obja}
\cL(R)=-\cV[R(X),Y]+\lambda\mathbb{D}(\mu_{R(X)}\Vert\gamma_{d}),
\end{equation}
where $\lambda \ge 0$ is a tuning parameter.
{\color{black} 
This parameter
provides a balance between the sufficiency property and the disentanglement constraint. A small $\lambda$ leads to a representation with more emphasis on sufficiency, while a large $\lambda$ yields a representation with more emphasis on disentanglement.}
We show in Theorem \ref{glmin} 
below that any $R^*$ satisfying (\ref{cida2}) is a minimizer of $\cL(R)$. Therefore, we can train a DDRM by minimizing an empirical version of $\cL(R)$.

There are several options for $\cV$ with the properties (a)-(c) described above. For example, we can take $\cV$ to be the mutual information.
However, in addition to estimating the DDRM $R$,
this choice requires nonparametric estimation of the ratio of the joint density and the marginal densities
 of
$Y$ and $R(X)$,  
 which is not an easy task.
 To be specific, in this work we use the distance covariance \citep{szekely2007measuring} between $Y$ and $R(X)$, which has an elegant $U$-statistic expression.
 It does not involve additional unknown quantities and is easy to compute. For the divergnce measure of two distributions, we use the
$f$-divergence \citep{ali1966general}, which includes the KL-divergence as a special case.

\subsection{Empirical objective function}
In this subsection, we formulate the objective function for the proposed deep dimension reduction method. We first describe some essentials about distance covariance and $f$-divergence.

\subsubsection{Distance covariance }
We  recall the concept of distance covariance \citep{szekely2007measuring}, which characterizes the dependence of two random variables.
%
Let $\mathfrak{i}$ be the imaginary unit $(-1)^{1 / 2}$. For any $\vt\in \mathbb{R}^{d}$ and $\vs\in \mathbb{R}^m$, let  $\psi_{Z}(\vt)=\mathbb{E} [\exp^{\mathfrak{i}\vt^TZ}], \psi_{Y}(\vs)=\mathbb{E} [\exp^{\mathfrak{i}\vs^TY}],$ and $\psi_{Z, Y}(\vt,\vs) = \mathbb{E} [\exp^{\mathfrak{i}(\vt^TZ+ \vs^TY)}]$ be the characteristic functions of random vectors $Z\in \real^d , Y \in \real^m,$ and the pair  $(Z, Y)$, respectively.
The squared distance covariance $\mathcal{V}[Z, Y]$ is defined as
$$
\mathcal{V}[Z, Y]=\int_{\mathbb{R}^{d+m}} \frac{\left|\psi_{Z, Y}(\vt, \vs)-\psi_{Z}(\vt) \psi_{Y}(\vs)\right|^2}{c_{d}c_{m}  \|\vt\|^{d+1}\|\vs\|^{q+1}} \mathrm{d} \vt \mathrm{d} \vs,
$$
where
$
c_{d}=\frac{\pi^{(d+1) / 2}}{\Gamma((d+1) / 2)}.
$
Given $n$ i.i.d copies $\{Z_{i}, Y_{i}\}_{i=1}^n$ of $(Z,Y)$,
an unbiased estimator of $\mathcal{V}$
is the empirical distance covariance $\widehat{\mathcal{V}}_{n}$,
which can be elegantly expressed as a $U$-statistic \citep{huo2016fast}
\begin{equation}\label{usta}
\widehat{\mathcal{V}}_{n}[Z, Y]
 = \frac{1}{C_{n}^{4}} \sum_{1 \leq i_{1}<i_{2}<i_{3}<i_{4} \leq n} h\left(\left(Z_{i_{1}}, Y_{i_{1}}\right), \cdots,\left(Z_{i_{4}}, Y_{i_{4}}\right)\right), \nonumber
\end{equation}
where $h$ is the kernel defined by
\begin{eqnarray*}\label{kernel}
h\left(\left(\vz_{1}, \vy_{1}\right),\ldots,
\left(\vz_{4}, \vy_{4}\right)\right)
&=&\frac{1}{4} \sum_{1 \leq i, j \leq 4 \atop i \neq j}\|\vz_{i}-\vz_{j}\| \|\vy_{i}-\vy_{j}\|
 +\frac{1}{24} \sum_{1 \leq i, j \leq 4 \atop i \neq j}\left\|\vz_{i}-\vz_{j}\right\| \sum_{1 \leq i, j \leq 4 \atop i \neq j}\|\vy_{i}-\vy_{j}\| \\
& &
 -\frac{1}{4} \sum_{i=1}^{4}(\sum_{1 \leq j \leq 4 \atop j \neq i}\left\|\vz_{i}-\vz_{j}\right\| \sum_{1 \leq j \leq 4 \atop i \neq j}\|\vy_{i}-\vy_{j}\|).
\end{eqnarray*}
For a categorial response $Y$ in multi-class classification problems, we can use one-hot vectors to code the classes, i.e., for the $k$th class, $Y$ is a unit vector with $k$th element equaling 1 and the remaining elements being 0 . The $L_2$ distance between two observed responses $y_i$ and $y_j$ is
$$
\left\|y_i-y_j\right\|_2=\left\{\begin{array}{cc}
0, & \text { if } y_i=y_j, \\
\sqrt{2}, & \text { if } y_i \neq y_j.
\end{array}\right.
$$
Note that the number $\sqrt{2}$ simply scales the whole objective function and does not affect the solution.
\subsubsection{$f$-divergence }
Let $\mu$ and $\gamma$ be two probability measures on $\mathbb{R}^{d}$. The   $f$-divergence \citep{ali1966general} between $\mu$ and $\gamma$  with $\mu \ll \gamma$ is defined  as
\begin{equation}\label{fdiv}
\mathbb{D}_f(\mu \Vert \gamma) = \int_{\mathbb{R}^{d}} f(\frac{\mathrm{d} \mu}{\mathrm{d} \gamma}) {\mathrm{d}} \gamma,
\end{equation}
where $f: \mathbb{R}^+ \rightarrow \mathbb{R} $ is a
differentiable convex function satisfying $f(1) = 0$.
Let   $f^*$ be the Fenchel conjugate of $f$ \citep{rockafellar1970convex},
defined by
\begin{equation}
\label{fdual}
f^*(t) = \sup_{x \in \real}\{ t x - f(x)\}, t\in \real.
\end{equation}
The  $f$-divergence (\ref{fdiv}) admits the following variational formulation \citep{keziou2003dual,nguyen2010estimating,nowozin2016f}.
\begin{lemma}\label{lem2}
Suppose that $f$ is s differentiable convex function. Then,
	\begin{equation}\label{fduala}
	\mathbb{D}_f(\mu \Vert \gamma) =  \max_{D:\mathbb{R}^{d}\rightarrow \mathrm{dom}(f^*)} \mathbb{E}_{Z\sim \mu} D(Z)-\mathbb{E}_{W\sim \gamma} f^*(D(W)),
	\end{equation}
	where $f^*$ is defined in (\ref{fdual}). In addition, the maximum is attained at $D(\vz) = f^{\prime}(\frac{\mathrm{d} \mu}{\mathrm{d} \gamma}(\vz)).$
\end{lemma}

Commonly used divergence measures include the Kullback-Leibler (KL) divergence, the Jensen-Shanon (JS) divergence and the $\chi^2$-divergence.
We summarize the details in Table~\ref{div}.
{\small
\begin{table}[ht!]
\caption{Three  examples of  $f$-divergence}
\label{div}
\vskip 0.15in
\begin{center}
\begin{footnotesize}
\begin{rm}
\begin{tabular}{lccc}
\toprule
$f$-Div		& $f(x)$ 								 & $ f^*(t)$ & $\mathbb{D}_f(\mu,\gamma) $   \\
\midrule
KL    		& $x \log x$ 							 & $ e^{t-1}$ &	  $\sup_D \{\Ebb_{Z\sim \mu} D(Z) -\Ebb_{W\sim \gamma} e^{D(W)-1}\}$     \\
JS 			& $-(x+1) \log \frac{x+1}{2} + x \log x$ & $-\log(2-\exp(t))$&    $\sup_D \{\Ebb_{Z\sim \mu} D(Z) +\Ebb_{W\sim \gamma} \log(2-\exp(D(W)))\}$ \\
$\chi^2$   	& $(x-1)^2$ 		                     & $t + \frac{t^2}{4}$& $\sup_D \{\Ebb_{Z\sim \mu} D(Z) -\Ebb_{W\sim \gamma} [D(W)+\frac{D^2(W)}{4}]\}$	  \\
\bottomrule
\end{tabular}
\end{rm}
\end{footnotesize}
\end{center}
\vskip -0.1in
\end{table}}
 To be specific, in this paper we use the KL divergence with $f(x)=x\log x$, which has the familiar form
$
 \mathbb{D}_{\text{KL}}(\mu \Vert \gamma) =
 \int_{\real^d} \left(\log \frac{\mathrm{d} \mu}{\mathrm{d} \gamma}\right) \mathrm{d} \mu.
$
The dual form of $f$ is
$f^*(t) = \exp(t-1).$  The variational representation
$
\mathbb{D}_{\text{KL}}(\mu \Vert \gamma)=\sup_D \{\Ebb_{Z\sim \mu} D(Z) -\Ebb_{W\sim \gamma} \exp({D(W)-1})\}.
$
{\color{black}
The
generative adversarial networks (GAN, \citet{goodfellow14})  corresponds to the JS-divergence.  Much work has been devoted to developing various extensions and alternative formulations of the original GAN \citep{li15,nowozin2016f,sutherland16,arjovsky17}.}


\subsubsection{Empirical objective function for DDR}
We are now ready to formulate an empirical objective function for learning DDRM.
Let $ R \in \cM $, where $\cM$ is the Maxwell class defined in (\ref{maxwell}).
By the variational formulation (\ref{fduala}), we can write the population version of the objective function
(\ref{obja}) as
\begin{equation}\label{lf}
\mathcal{L}(R) = -{\mathcal{V}}[R(X),Y]
+  \lambda  \max_{D}
\{\mathbb{E}_{X\sim \mu_{X}}D(R(X))- \mathbb{E}_{W\sim \gamma_{d}}f^*(D(W))\}.
\end{equation}
This expression is convenient since we can simply replace
the  expectations 
by the corresponding empirical averages.

\begin{theorem}\label{glmin}
	We have $R^* \in \arg \min_{R\in \cM } \mathcal{L}(R)$ provided (\ref{cida2}) holds.
\end{theorem}

According to Theorem  \ref{glmin}, it is natural to estimate $R^*$ based on the empirical version of the objective function (\ref{lf}) when a random sample $\{(X_i,Y_i)\}_{i=1}^n$ is available.

We estimate $R^*$ nonparametrically using feedforward neural networks (FNN) \citep{Schmidhuber_2015}.
Two networks are employed: the representer network $R_{\vtheta}$ with parameter $\vtheta$ for  estimating $R^*$ and a second network $D_{\vphi}$  with parameter $\vphi$ for estimating the discriminator $D$.
For any function $f(\vx): {\mathcal X} \to \real^d$,  denote $\|f\|_{\infty} = \sup_{\vx \in {\mathcal{X}}} \|f(\vx)\|$, where $\|\cdot\|$ is the Euclidean norm.

\begin{itemize}
\item Representer network $R_{\vtheta}$:  This network is used for
training $R^*$. Let
	$\mathbf{R}\equiv \mathbf{R}_{\mathcal{H}, \mathcal{W}, \mathcal{S}}$
 be the set of such ReLU
 neural networks  $R_{\vtheta}: \real^p \rightarrow \real^{d}.$
 with parameter $\vtheta$,  depth   $\mathcal{H}$,   width $\mathcal{W}$,  size $ \mathcal{S}$.
 Here the depth $\mathcal{H}$ refers to the number of hidden layers, so the network has $\cH+1$ layers in total. A $(\cH+1)$-vector $(w_0, w_1, \ldots, w_{\cH})$
specifies the width of each layer, where $w_0=p$ is the dimension of the input data and $w_{\cH}=d$ is the dimension of the output. The width $\cW=\max\{w_1, \ldots, w_{\cH}\}$ is the maximum width of the
hidden layers. The size $\cS=\sum_{i=0}^{\cH}[w_i\times (w_i+1)]$ is the total number of parameters in the network.

\item  Discriminator network $D_{\vphi}$: This network is used as the witness function for  checking whether the distribution of the estimator of $R^*$ is approximately the same as  $\mathcal{N}(\m0, \mI_{d})$.
    Similarly, denote $\mathbf{D}\equiv \mathbf{D}_{\tilde{\mathcal{H}}, \tilde{\mathcal{W}}, \tilde{\mathcal{S}}}$
as the set of ReLU
neural networks  $D_{\vphi}: \mathbb{R}^{d}\rightarrow \mathbb{R}$ with
parameter $\vphi$, depth   $\tilde{\mathcal{H}}$,   width $\tilde{\mathcal{W}}$,  size $\tilde{\mathcal{S}}$.
\end{itemize}

Let $\{W_i\}_{i=1}^n$ be  $n$ i.i.d random vectors drawn from $\gamma_{d}$.
The estimated DDRM is defined by
\begin{equation}\label{dnpe}
\widehat{R}_{\vtheta} \in\arg\min\limits_{R_{\vtheta} \in \mathbf{R}}
\widehat{\mathcal{L}}(R_{\vtheta})
\end{equation}
where
$
\widehat{\mathcal{L}}(R_{\vtheta})
= -\widehat{\mathcal{V}}_{n}[R_{\vtheta}(X), Y]+\lambda \widehat{\mathbb{D}}_{f}(\mu_{R_{\vtheta}(X)} \Vert \gamma_{d}).
$
Here $\widehat{\mathcal{V}}_{n}[R_{\vtheta}(X), Y]$ is an unbiased and consistent estimator of $ \mathcal{V}[R_{\vtheta}(X), Y]$ as defined
in (\ref{usta}) based on $\{(R_{\vtheta}(X_i), Y_i), i=1, \ldots, n\}$
and
\begin{equation}\label{es2}
\widehat{\mathbb{D}}_{f}(\mu_{R_{\vtheta}(X)} \Vert \gamma_{d})
 = \max_{D_{\vphi} \in \mathbf{D}}
\frac{1}{n} \sum_{i=1}^n [D_{\vphi}(R_{\vtheta}(X_i))- f^*(D_{\vphi}(W_i))].
\end{equation}

This objective function consists of two terms:
(a) the term $\lambda \widehat{\mathcal{V}}_{n}[R_{\vtheta}(X), Y]$
 is an unbiased and consistent estimator of $\lambda \mathcal{V}[R_{\vtheta}(X), Y]$, which is a  measure that quantifies  the conditional independence $X\indep Y|R_{\vtheta}(X)$;
(b) the term  $\widehat{\mathbb{D}}_{f}(\mu_{R_{\vtheta}(X)} \Vert \gamma_{d})$
promotes disentanglement among the components of $R_{\vtheta}(X)$ by encouraging $R_{\vtheta}(X)$ to be distributed as $N(0, \bI_d)$. This is the dual form of the  $f$-GAN  loss \citep{goodfellow14, nowozin2016f}. We note that GANs seek to find a map from a reference distribution such as Gaussian to the data space, here we do the reverse and try to find a representation of the data to be distributed like a reference distribution.

\section{Consistency}
\label{theory}
We establish the consistency of the estimated DDRM in the sense that the
excess risk $\mathcal{L}(\widehat{R}_{\vtheta} ) -\mathcal{L}({R}^*) $ converges to zero, where $\widehat{R}_{\vtheta} $ is  the deep nonparametric estimator  in (\ref{dnpe}).
It is clear that to achieve consistency, it is necessary to require the network parameters to increase as the sample size increases. This is similar to requiring the bandwidth of a nonparametric kernel density estimator to  depend on the sample size. There is an extensive literature on how to select the bandwidth parameter in nonparametric density estimation problems. How to
choose the structure parameters of a neural network is a more complicated problem. To the best of our knowledge, it has not been systematically studied in the literature. We provide a particular specification below that ensures the consistency of the estimated representation. However, this specification is not necessarily optimal, it only represents our first attempt to tackle this difficult problem.

We make the following basic assumptions about the target parameter and the model.
\begin{enumerate}
	\item[(A1)]
{\color{black}
The target representation $R^*$ is Lipschitz continuous
with Lipschitz constant $L_1$}.
	\item[(A2)]
 For 
 every  $R \in \mathbf{R}\equiv \mathbf{R}_{\mathcal{H}, \mathcal{W}, \mathcal{S}}$,
  we assume the density ratio $r(\vz)=\frac{\mathrm{d}\mu_{R(X)}}{\mathrm{d}\gamma_{d}}(\vz)$ to be Lipschitz continuous
  with Lipschitz constant $L_2$, and $c_1\leq r(\vz)\leq c_2$ for some constants $0< c_1 \le c_2 < \infty$.
 \item[(A3)]
    {\color{black}   $\mathrm{supp}(\mu_{X})$ is contained in a compact set, say $[-B_1,B_1]^p$ with a finite $B_1$ and denote its density function as $f_{X}(x)$. $Y$ is bounded almost surely, say $\|Y\|\leq C_1$ a.s..}
\end{enumerate}

{\color{black}

Let  $B_2 = \max\{|f^{\prime}(c_1)|,|f^{\prime}(c_2)|\}$ and  $B_3 = \max_{|s|\leq  2 B_2} |f^*(s)|$. For the KL-divergence, we have
$B_2=\max\{\log c_1, \log c_2\} +1$  and
$B_3=\exp(2B_2).$
We specify the network parameters of the representer  $R_{\vtheta}$  and the discriminator  $D_{\vphi}$ as follows.
\begin{enumerate}
	\item[(N1)]
	Representer network $\mathbf{R}\equiv \mathbf{R}_{\mathcal{H}, \mathcal{W}, \mathcal{S}}$ parameters:
	depth $\mathcal{H} = \mathcal{O}( \log n)$ width
$\mathcal{W} = \mathcal{O}(n^{\frac{p}{2(2+p)}}/\log n),$  size $\mathcal{S} =\mathcal{O}(dn^{\frac{p}{2+p}}/\log^4 (npd)),$
and $\|R\|_{L^{\infty}} \leq  2\|R^*\|_{L^{\infty}}, \forall R \in \mathbf{R}.$
	\item[(N2)] Discriminator network $\mathbf{D}\equiv \mathbf{D}_{\tilde{\mathcal{H}}, \tilde{\mathcal{W}}, \tilde{\mathcal{S}}}$ parameters:
depth $\tilde{\mathcal{H}} = \mathcal{O}(\log n),$ width $\tilde{\mathcal{W}} = \mathcal{O}(n^{\frac{d}{2(2+d)}}/\log n),$ size $\tilde{\mathcal{S}} =\mathcal{O}(n^{\frac{d}{2+d}}/\log^4 (npd)),$
and
$\|D\|_{L^{\infty}} \leq   2B_2, \forall D \in \mathbf{D}.$
\end{enumerate}
}
We again note that these specifications of the network parameters are not necessarily unique or optimal. Our goal here is to provide theoretical support for the proposed method in the sense that there exist networks with the above specifications leading to the consistency of the estimated representation map.

\begin{theorem}\label{esterr}
Set 
$\lambda = \mathcal{O}(1)$.  Suppose conditions (A1)-(A3) hold and set the network parameters according to  (N1)-(N2). Then
$
	\mathbb{E}_{\{X_i,Y_i,{W}_i\}_{i=1}^n} [\mathcal{L}(\widehat{R}_{\vtheta} ) -\mathcal{L}({R}^*)] \to 0.
$
\end{theorem}
The proof of this theorem is given in the appendix.
Conditions (A1) and (A2) are regularity conditions that are often assumed in nonparametric estimation problems. 
The result established in Theorem \ref{esterr} shows that the learned DDRM achieves asymptotic sufficiency under the conditions (A1) and (A2) and with the specifications
(N1) and (N2) for the network parameters.


{\color{black}
There have been intensive efforts devoted to understanding the theoretical properties of deep neural network models in recent years.
Several stimulating papers have studied the statistical convergence properties of nonparametric regression using neural networks
\citep{bauer2019deep, schmidt2020nonparametric,
farrell2021deep, jiao2021deep}.
There have also been some recent works on the non-asymptotic error bounds of GANs.
For example, 
\citet{zhang2018on} considered the generalization error of GANs.
\citet{liang2020} studied  the rates of convergence for learning distributions implicitly
with GAN under several forms of the integral probability metrics.
\citet{bai2018} analyzed the estimation error bound of GANs under the Wasserstein distance for a special class of distributions implemented by a generator.
\citet{zhao2020} studied the convergence rates of GAN distribution estimators when both the evaluation class and the target density class are H\"older classes.

In the present problem,  the objective function (\ref{dnpe})
is the combination of a loss that is a $U$-process indexed by a class of neural networks and a GAN-type loss indexed by two classes of neural networks. This objective function is more complicated than the least squares loss or the GAN loss analyzed in the aforementioned works. Therefore, the problem here is more difficult. To the best of our knowledge, the consistency property of the excess risk of the minimizer of such an objective function has not been analyzed in the literature.
}

\section{Computation}
\label{algorithm}

Lemma \ref{lem2} implies that  training  of  $\vphi$ with fixed $\vtheta$ 
is to push forward the distribution of $R(X)$ to the reference distribution $\gamma_{d}=N(0, \bI_{d})$. For this purpose, we need to estimate an optimal discriminator $D_{\vphi}$  approximating the optimal dual function $D(\vz) = f^{\prime}(r(\vz))$, where
$r(\vz)$
is the ratio for  the density of $\mu_{R_{\vtheta}(\vx)}$ over the density of $\gamma_{d}$.
Note that $f^{\prime}$ is a strictly increasing function if $f$ is strictly convex, which is true for all the commonly used divergence measures. Thus the problem of estimating the discriminator is essentially that of estimating the density ratio. Therefore,
in our implementation, we utilize  the computationally stable  particle method based on gradient flow in probability measure spaces \citep{yuan2019deep, gao2020learning}.
The key idea of this particle method is
to seek a sequence of nonlinear but simpler residual maps,   $\mathbb{T}(\vz) = \vz+s \mathbf{v}(\vz)$, where $s>0$ is a small step size,
pushing the samples from $\mu_{R_{\vtheta}(\vx)}$ to the target distribution $\gamma_{d}$ along a velocity fields $\mathbf{v}(\vz) 
= -\nabla f^{\prime}(r(\vz))$ that most decreases the $f$-divergence $\mathbb{D}_f(\cdot||\gamma_{d^*})$  at $\mu_{R_{\vtheta}(X)}$ \citep{yuan2019deep}.  The residual maps can be estimated via deep density-ratio estimation. Specifically, the estimated residual maps take the form
$
\mathbb{T}(\vz)=\vz+s \widehat{\mathbf{v}}(\vz),\  \vz
\in \real^{d},
$
where
$\widehat{\mathbf{v}}(\vz)=-\nabla f^{\prime}(\hat{r}(\vz)).$
Here $\hat{r}(\vz)$ is an estimated density ratio of the density of $R_{\vtheta}(\vx)$ at the current value of $\vtheta$ over the density of the reference distribution. The estimator $\hat{r}(\vz)$ is constructed as follows.
Let  $Z_i = R_{\vtheta}(X_i)$ and generate $W_i \sim \gamma_{d},  i =1,2,\ldots, n$. We solve
\begin{equation}
\label{dra}
\widehat{D}_{\vphi} \in \arg \min_{D_{\vphi}} \frac{1}{n}
\sum_{i=1}^{n}\left\{\log [1+\exp(D_{\vphi}(Z_{i}))]+ \log [1+\exp(-D_{\vphi}(W_{i}))]\right\}
\end{equation}
with stochastic gradient descent (SGD). Then the estimated density ratio
$\hat{r}(\vz) = \exp(-\widehat{D}_{\vphi}(\vz))$.
Here we note that the population version of the loss function in (\ref{dra})
 is minimized at
$-\log(r(\vz))$. Therefore,
$\widehat{D}_{\vphi}(\vz) $
in (\ref{dra}) provides a good estimator of $-\log(r(\vz))$.
See \cite{gao2020learning} for a detailed description of this particle approach. Here, we use this approach
to transform  $Z_i = R_{\vtheta}(X_i), i = 1,\ldots, n$ into Gaussian samples (we still denote them as $Z_i$) directly. Once this is done, we update $\vtheta$ via minimizing the loss $$\frac{1}{n}\sum_{i=1}^n\|R_{\vtheta}(X_i)-Z_i\|^2-\lambda \widehat{\mathcal{V}}{n}[R_{\vtheta}(X), Y].$$

We depict the DDR algorithm in the flowchart in Figure \ref{flc}
and give a detailed description below.
\textbf{Pseudo-code for the DDR algorithm}
\begin{itemize}
\setlength{\itemsep}{-5 pt}
	\item Input $\{X_i,Y_i\}_{i=1}^n$.
	Tuning parameters:
	$s, \lambda, d$. Sample  $\{W_i\}_{i=1}^n \sim \gamma_{d}$.
	\item \textit{Outer loop for $\vtheta$}
\vspace{-0.2 cm}
	\begin{itemize}
\setlength{\itemsep}{-3 pt}
		\item \textit{Inner loop (particle method)}
\vspace{-0.2 cm}
		\begin{itemize}
\setlength{\itemsep}{-3 pt}
			\item Let  $Z_i = R_{\vtheta}(X_i),   i =1,\ldots,,n$.
			\item Solve  $\widehat{D}_{\vphi} \in \arg \min_{D_{\vphi}} 
		\frac{1}{n}	\sum_{i=1}^{n}\left\{\log [1+\exp(D_{\vphi}({Z}_{i}))]+ \log [1+\exp(-D_{\vphi}(W_{i}))]\right\}.$
			\item  Define the residual map $\mathbb{T}(\vz)=\vz - s \nabla f^{\prime}(\hat{r}(\vz))$
with $\hat{r}({\vz})=\exp(-\widehat{D}_{\vphi}(\vz)).$
			\item  Update the particles  $Z_i = \mathbb{T}(Z_i)$,   $i =1,2,...,n$.
		\end{itemize}
\vspace{-0.2 cm}
		\item \textit{End inner loop}
		\item Update $\vtheta$   via minimizing
$ -\widehat{\mathcal{V}}_{n}[R_{\vtheta}(X), Y] +\lambda\sum_{i=1}^n\|R_{\vtheta}(X_i)-Z_i\|^2/n$ using SGD.
	\end{itemize}
	\item \textit{End outer loop}
\end{itemize}
\begin{figure}[htbp]
\centering
	\includegraphics[width=0.7\columnwidth]{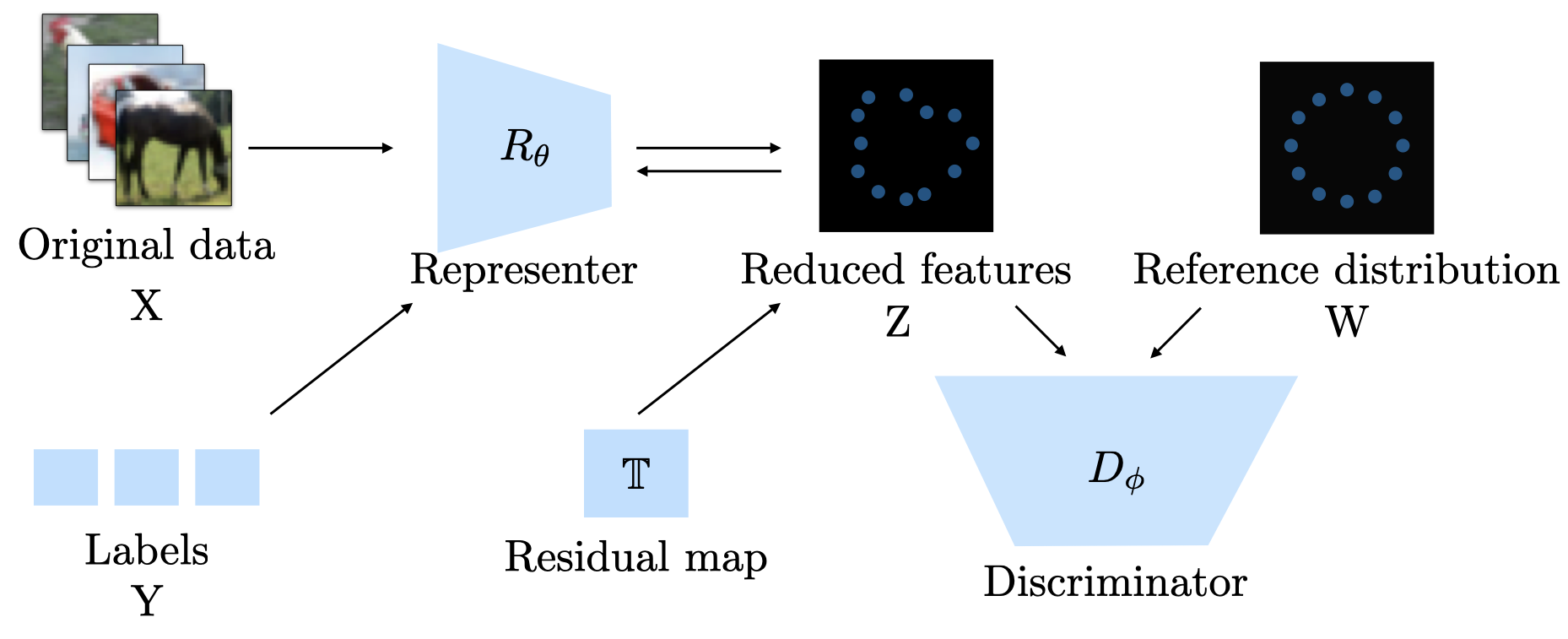} 
	\caption{Flow chart for deep dimension reduction (DDR) }
	\label{flc}
\end{figure}
\section{Numerical experiments}
\begin{table}[htbp]
\caption{\color{black}Summary information of DDR and compared methods. $X^l$ and $X^u$ represent labeled and unlabeled data, respectively, while $Y^l$ represents labeled targets.}
\label{tab:methods}
\centering
\resizebox{\columnwidth}{!}{%
\begin{tabular}{@{}llccc@{}}
\toprule
Method   & Name                                            & Input                             & Supervision     & Based Model     \\\midrule
DDR      & Deep Dimension Reduction                        & $X^{l}, Y^{l}$                    & Supervised      & Neural networks \\
NN       & Neural Networks                                 & $X^{l}, Y^{l}$                    & Supervised      & Neural networks \\
dCorAE   & Distance Correlation Autoencoder                & $X^{l}, Y^{l}$                    & Supervised      & Neural networks \\
OLS      & Ordinary Least Squares                          & $X^{l}, Y^{l}$                    & Supervised      & Linear          \\
SIR      & Sliced Inverse Regression                       & $X^{l}, Y^{l}$                    & Supervised      & Linear          \\
SAVE     & Sliced Average Variance Estimation              & $X^{l}, Y^{l}$                    & Supervised      & Linear          \\
GSIR     & Generalized Sliced Inverse Regression           & $X^{l}, Y^{l}$                    & Supervised      & Kernel          \\
GSAVE    & Generalized Sliced Average Variance Estimation  & $X^{l}, Y^{l}$                    & Supervised      & Kernel          \\
Semi-VAE & Semi-supervised Variational Autoencoders        & $X^{l}, Y^{l}$ and $X^{u}$ & Semi-supervised & Neural networks \\
InfoVAE  & Information Maximizing Variational Autoencoders & $X^{l}, Y^{l}$ and $X^{u}$ & Semi-supervised & Neural networks \\
PCA      & Principal Component Analysis                    & $X^{l}$ and $X^{u}$               & Unsupervised    & Linear          \\
SPCA     & Sparse Principal Component Analysis             & $X^{l}$ and $X^{u}$               & Unsupervised    & Linear       \\  \bottomrule
\end{tabular}}
\vskip -0.1in
\end{table}
\label{experiments}
We evaluate the performance of DDR using simulated and benchmark real data.
Since DDR is not trying to estimate a classifier or a regression function directly, but rather to learn a representation with the desired properties of sufficiency, low-dimensionality and disentanglement, we design the experiments to evaluate the performance of the learned representations based on DDR in terms of prediction when using these representations. The results demonstrate that a simple classification or regression model using the learned representations performs better than or comparably with the best classification or regression methods using deep neural networks.
Details on the network structures and hyperparameters are included in the appendix.
{\color{black} Summary information of DDR and compared methods, including the names of methods, their input, learning types, and models of methods, is given in Table \ref{tab:methods}.}
Our experiments were conducted on Nvidia DGX Station workstation using a single Tesla V100 GPU unit.

\subsection{Simulated data}
In this subsection, we evaluate DDR on simulated regression and classification problems.

\noindent
{\color{black}\textbf{Regression \uppercase\expandafter{\romannumeral1}.}
We generate $10,000$ data points from two models:
\begin{enumerate}[ ]
\setlength{\itemsep}{-5 pt}
	\item Model (a). \(Y={\rx_{1}}{[0.5+\left(\rx_{2}+1.5\right)^{2}]^{-1}}+\left(1+\rx_{2}\right)^{2}+\sigma \vveps \), where \(X \sim N\left(\m0, {\mI}_{20}\right)\);
	\item Model (b).  \(Y=\sin ^{2}\left(\pi X_{1}+1\right)+\sigma \vveps\), where \(X \sim \mathrm{Uniform}[0,1]^{20}\).
\end{enumerate}
In both models, \(\vveps\sim N\left(\m0, {\mI}\right)\).
We use  a  3-layer network 
with ReLU activation for $R_{\vtheta}$ and a single hidden layer ReLU network for $D_{\vphi}$.
We compare DDR with four prominent 
sufficient dimension reduction methods:   sliced inverse regression (SIR) \citep{li1991sliced}, sliced average variance estimation (SAVE) \citep{cook1991sliced},  generalized sliced inverse regression (GSIR) and generalized sliced average variance estimation (GSAVE) \citep{lee2013gen, li2018sufficient}.
{\color{black} SIR slices the range of $Y$ and obtains the crude estimation of the inverse regression $E({X}|Y)$.  
Then the eigenvectors of the covariance matrix $Cov(E({X}|Y))$ that lie in the central subspace of data can be estimated via weighted PCA. SIR is a first moment method to estimate the central subspace from $E({X}|Y)$, while SAVE is a second moment method to estimate the space from $Var(E({X}|Y))$ that is primarily used to solve symmetric data problems. 
Similarly, SAVE also utilizes the weighted PCA to estimate eigenvectors that lie in the central subspace. 
GSIR and GSAVE are generalized versions of SIR and SAVE, respectively. 
Both of them estimate central subspace 
in the reproducing kernel Hilbert space (RKHS) instead of using the covariance matrix in both SIR and SAVE.}
Also, we compare DDR with two deep learning based methods: neural networks (NN) with least square (LS) loss as the last layer, denoted as NN+LS,  and distance correlation autoencoder (dCorAE) \citep{wang2018distance}. 
{\color{black} dCorAE targets at two objectives for both reconstruction and classification, presenting a trade-off between two tasks during training.}

\begin{table}[htpb]
	\caption{{\color{black}Average prediction errors  and their standard errors (based on 5-fold validation)}}
	\label{tab:sim}
	\begin{center}
			\scriptsize
			\begin{tabular}{lcccccc}
				\toprule
				&\multicolumn{3}{c}{ Model  (a)   }&\multicolumn{3}{c}{Model  (b)  } \\
				\cmidrule(r){2-4}\cmidrule(r){5-7}
				Method     & $\sigma = 0.1$ & $\sigma = 0.4$ &$\sigma = 0.8$& $\sigma  = 0.05$ & $\sigma = 0.1$ & $\sigma = 0.2$  \\
				\midrule
DDR &	\textbf{0.127	 \(\pm\) .005}	&	\textbf{0.555	 \(\pm\) .010}&	\textbf{1.088	 \(\pm\) .009}& \textbf{0.052	 \(\pm\) .001} 	& \textbf{0.105	 \(\pm\) .003}	&	\textbf{0.241	 \(\pm\) .010}\\
				NN+LS	&	0.147	 \(\pm\) .028	&	0.575	 \(\pm\) .008	&	1.150	 \(\pm\) .013	&	0.053	 \(\pm\) .001	& 0.107	 \(\pm\) .002	&	0.242	 \(\pm\) .010\\
				dCorAE	&		0.153 \(\pm\) .015	&		0.549 \(\pm\) 	.012&		1.101 \(\pm\) .015 &	0.065	 \(\pm\) .001	& 0.135	 \(\pm\) .001	&	0.275	 \(\pm\) .004\\
				SIR	&	1.484	 \(\pm\) .047	&	1.599	 \(\pm\) .050	&	1.712	 \(\pm\) .037	&	0.252	 \(\pm\) .002	&0.268	 \(\pm\) .002	&	0.323	 \(\pm\) .005	\\
				SAVE	&	1.482	 \(\pm\) .048	&	1.588	 \(\pm\) .049	&	1.715	 \(\pm\) .038	&0.252	 \(\pm\) .002	&	0.268	 \(\pm\) .003	&	0.323	 \(\pm\) .005\\
				GSIR	&	1.477	 \(\pm\) .047	&	1.598	 \(\pm\) .050	&	1.707	 \(\pm\) .039	&0.267	 \(\pm\) .004	&	0.269	 \(\pm\) .004	&	0.322	 \(\pm\) .006\\
				GSAVE	&	1.478	 \(\pm\) .048	&	2.602	 \(\pm\) .079	&	2.654	 \(\pm\) .041	&0.265	 \(\pm\) .003	&	0.267	 \(\pm\) .004	&	0.339	 \(\pm\) .006\\
				\bottomrule
			\end{tabular}
	\end{center}
	\vskip -0.1in
\end{table}
We fit a linear model with the learned features and the response variable, and report the  prediction errors in Table \ref{tab:sim}.
We see that DDR outperforms SIR, SAVE, GSIR, GSAVE, NN+LS and dCorAE in terms of prediction error.}

%


\noindent
\textbf{Regression \uppercase\expandafter{\romannumeral2}.}
We generate $5000$ data points from three simulated models:
\begin{enumerate}[ ]
\setlength{\itemsep}{-5 pt}
\item Model (a).
  \(Y=\left(\rx_{1}+\rx_{2}\right)^{ 2} + (1+\exp({\rx_1}))^2+\vveps\);
\item Model (b).
 \(Y=\sin \left({\pi\left(\rx_{1}+\rx_{2}\right)}/{10}\right)+\rx_1^2+\vveps\);
\item Model (c).
 \(Y=\left(\rx_{1}^{2}+\rx_{2}^{2}\right)^{{1}/{2}} \log \left(\rx_{1}^{2}+\rx_{2}^{2}\right)^{{1}/{2}}+\vveps\),
\end{enumerate}
where $\vveps \indep X$ and $ \vveps \sim N(\m0, 0.25\cdot{\mI}_{10}).$
For the distribution of the 10-dimensional predictor \(X\), we consider three
scenarios:
Scenario (\romannumeral1): \(X \sim N\left(\m0, {\mI}_{10}\right)\);
independent Gaussian predictors;
 Scenario (\romannumeral2): \(X \sim \frac{1}{3}N\left(-2\cdot\textbf{1}_{10}, \mI_{10}\right)+\frac{1}{3} \mathrm{Uniform}[-1,1]^{10}+\frac{1}{3}N\left(2\cdot\textbf{1}_{10}, \mI_{10}\right) \), independent non-Gaussian predictors;
Scenario (\romannumeral3): \(X \sim N\left(\m0,0.3\cdot{\mI}_{10}+0.7\cdot\textbf{1}_{10} \textbf{1}_{10}^{\top}\right)\).
 correlated Gaussian predictors.
These models and the distributional scenarios are modified from \citep{lee2013gen, li2018sufficient}.

{\setlength\tabcolsep{1.5pt}
\begin{table}[htpb]
	\caption{\color{black}Average prediction errors (APE), distance correlation (DC), conditional Hilbert-Schmidt independence criterion (HSIC)  and their standard errors (based on 5-fold validation)}
	\label{tab:sim2}
	\begin{center}
   	 \tiny
        \begin{tabular}{llccccccccc}
            \toprule
            & & \multicolumn{3}{c}{Model  (a)  }                                                 & \multicolumn{3}{c}{Model  (b)}                    & \multicolumn{3}{c}{Model  (c)}                                                 \\
            \cmidrule(r){3-5}\cmidrule(r){6-8}\cmidrule(r){9-11}
                     & Method & APE& DC &HSIC& APE& DC &HSIC& APE& DC &HSIC\\
            \midrule
            \multirow{3}{*}{Scenario (\romannumeral1)} & DDR    & \textbf{6.1      \(\pm\) 3.5}    & \textbf{1.0      \(\pm\) .0 }     & \textbf{34.7      \(\pm\) 3.5}    & \textbf{0.3      \(\pm\) .0 }     & \textbf{1.0       \(\pm\) .0 }    & \textbf{49.1      \(\pm\)7.7}    & \textbf{ 0.3       \(\pm\) .0 }     & \textbf{0.9       \(\pm\) .0 }     & \textbf{36.4     \(\pm\) 2.8 }    \\
                               & GSIR   & 28.3     \(\pm\) 7.5    & 0.2       \(\pm\) .0      & 64.4      \(\pm\) 2.9    & 1.4      \(\pm\) .1    & 0.1       \(\pm\) .0      & 134.8     \(\pm\)13.5   & 0.8       \(\pm\) .0      & 0.2      \(\pm\) .0       &\text{ 67.4     \(\pm\) 3.9}     \\
                               & GSAVE  & 28.3     \(\pm\) 7.4    & 0.1       \(\pm\) .0      & \text{72.1      \(\pm\) 4.1 }   & 1.4      \(\pm\) .0      & 0.1        \(\pm\) .0    &\text{ 175.9     \(\pm\)6.4 }   & 0.8      \(\pm\) .0      & 0.2       \(\pm\) .0       & 66.5     \(\pm\) 2.8     \\
                               \cmidrule(r){1-11}
            \multirow{3}{*}{Scenario (\romannumeral2)}  & DDR    & 1\textbf{83.1    \(\pm\) 98.7 }  &\textbf{ 0.9      \(\pm\) .1  }   & \text{45.1      \(\pm\) 3.3 }   & \textbf{0.4      \(\pm\) .1 }   & \textbf{1.0       \(\pm\) .0 }    &\textbf{ 16.8      \(\pm\)2.0 }     &\textbf{ 0.3      \(\pm\) .1 }   & \textbf{1.0      \(\pm\) .0 }     & \textbf{8.6      \(\pm\) 0.7}     \\
                               & GSIR   & 664.6    \(\pm\) 38.1   & 0.1      \(\pm\) .0     & \textbf{43.9      \(\pm\) 2.6}    & 3.3      \(\pm\) .2    & 0.1       \(\pm\) .0     & 27.6      \(\pm\)1.7    & 1.5      \(\pm\) .1    & 0.6      \(\pm\) .0       & 14.6     \(\pm\) 0.8     \\
                              & GSAVE  & 662.0    \(\pm\) 38.0     & 0.0      \(\pm\) .0      & \text{48.0      \(\pm\) 2.7 }   & 3.2      \(\pm\) .2    & 0.2        \(\pm\) .0     & \text{32.0      \(\pm\)1.5 }   & 2.4      \(\pm\) .0      & 0.0       \(\pm\) .0     & \text{15.5     \(\pm\) 0.6 }    \\
                                \cmidrule(r){1-11}
            \multirow{3}{*}{Scenario (\romannumeral3)} & DDR    & \textbf{12.5     \(\pm\) 11.1 }  & \textbf{0.8      \(\pm\) .3  }   & \textbf{37.0      \(\pm\) 4.7 }   & \textbf{ 0.3      \(\pm\) .0 }     & \textbf{1.0        \(\pm\) .0  }   &\textbf{ 48.6      \(\pm\)7.1 }   & \textbf{ 0.3      \(\pm\) .1 }   & \textbf{0.9      \(\pm\) .1}     & \text{36.8     \(\pm\) 5.9}     \\
                               & GSIR   & 32.2     \(\pm\) 6.1    & 0.2      \(\pm\) .1     & \text{61.4      \(\pm\) 6.0 }     & 1.0      \(\pm\) .1    & 0.2       \(\pm\) .0     & 51.3      \(\pm\)51.3   & 0.6      \(\pm\) .0      & 0.6      \(\pm\) .0      & \textbf{25.8     \(\pm\) 25.6}    \\
                               & GSAVE  & 31.8     \(\pm\) 6.4    & 0.2      \(\pm\) .1     & 60.0      \(\pm\) 4.4    & 1.0      \(\pm\) .0      & 0.3        \(\pm\) .0      &\text{ 119.3     \(\pm\)8.8 }   & 0.6      \(\pm\) .0      & 0.7      \(\pm\) .0       & \text{54.3     \(\pm\) 3.8}\\
                               \bottomrule
        \end{tabular}
	\end{center}
	\vskip -0.1in
\end{table}
}
We compare DDR with generalized sliced inverse regression (GSIR) and generalized sliced average variance estimation (GSAVE) \citep{lee2013gen, li2018sufficient}.
In DDR, we adopt a 4-layer network for $R_{\vtheta}$ and a 3-layer network for $D_{\vphi}$ with Leaky ReLU activation. For all methods, we fit a linear model with the learned features and the response variable, and report the prediction error, distance correlation between representation and the response variable, and conditional Hilbert-Schmidt independence criterion (HSIC) \citep{fukumizu2008kernel}. 
The results are presented in Table \ref{tab:sim2}.
{\color{black}The representations learned with DDR present higher distance
correlations with the response and lower conditional HSICs, suggesting that DDR is capable of better capturing data information and conditional independence property than other methods.
Moreover, DDR significantly outperforms GSIR and GSAVE in terms of prediction errors in all scenarios.}

\medskip
\noindent
\textbf{Classification.}
We visualize the learned features of DDR on three simulated datasets.
We first generate (1) 2-dimensional  concentric circles from two classes as in Figure \ref{fig:latent} (a);
(2) 2-dimensional moons data from  two classes as in Figure \ref{fig:latent} (e);
(3)  3-dimensional  Gaussian data from six classes as in Figure \ref{fig:latent}  (i).
In each dataset, we generate 5,000 data points for each class.
We next map the data  into 100-dimensional space using matrices with  entries  i.i.d $\mathrm{Unifrom}([0,1])$.
Finally, we apply DDR to these $100$-dimensional datasets to learn $2$-dimensional features.
We use a 10-layer dense convolutional network (DenseNet) \citep{huang2017densely} as  $R_{\vtheta}$ and a 4-layer
network 
as $D_{\vphi}$.
We display the evolutions of the learned 2-dimensional features by DDR in Figure \ref{fig:latent}.  For ease of visualization, we push all the distributions onto the uniform distribution on the unit circle, which is done by normalizing the standard Gaussian random vectors to length one.
Clearly, the learned features for different classes in the
examples are well disentangled.
\begin{figure} [htpb]
	\centering
	\begin{tabular}{cccc}
		\includegraphics[width=0.1\columnwidth]{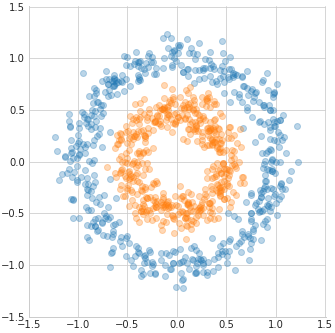}&
		\includegraphics[width=0.1\columnwidth]{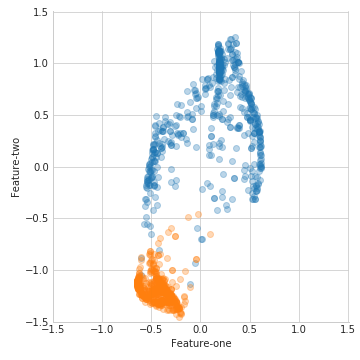} &
		\includegraphics[width=0.1\columnwidth]{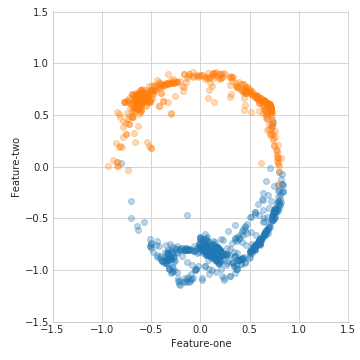} &
		\includegraphics[width=0.1\columnwidth]{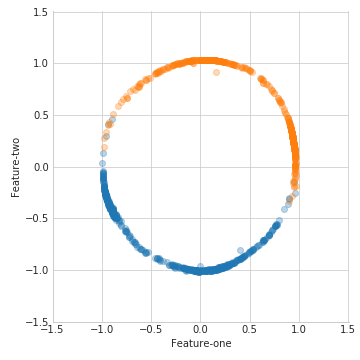} \\
{\scriptsize	(a) Epoch = 0} & {\scriptsize (b)  10} & {\scriptsize (c) 30 }&{ \scriptsize(d) 500} \\
	\end{tabular}
	\begin{tabular}{cccc}
		\includegraphics[width=0.1\columnwidth]{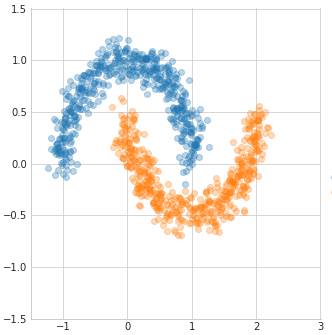}&
		\includegraphics[width=0.1\columnwidth]{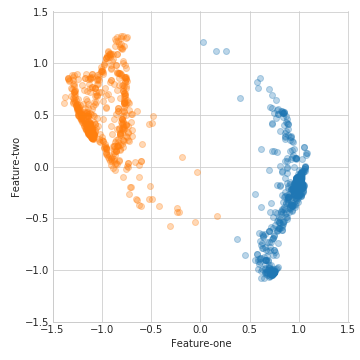} &
		\includegraphics[width=0.1\columnwidth]{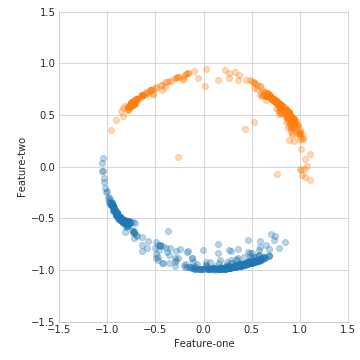} &
		\includegraphics[width=0.1\columnwidth]{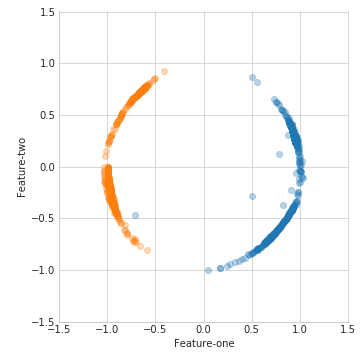} \\
{\scriptsize (e) Epoch = 0 }& {\scriptsize (f)  10}&{\scriptsize (g) 30 }& {\scriptsize (h) 500 }\\
	\end{tabular}
	\begin{tabular}{cccc}
		\includegraphics[width=0.1\columnwidth]{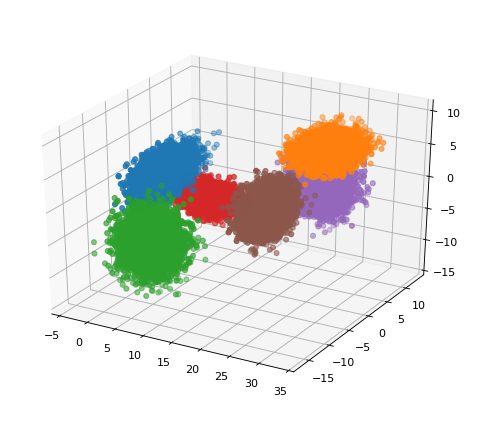}&
		\includegraphics[width=0.1\columnwidth]{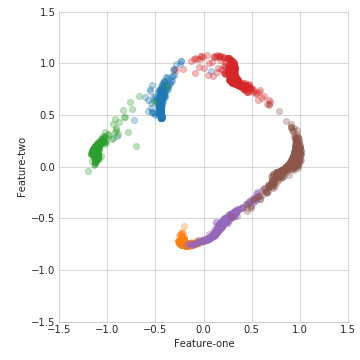} &
		\includegraphics[width=0.1\columnwidth]{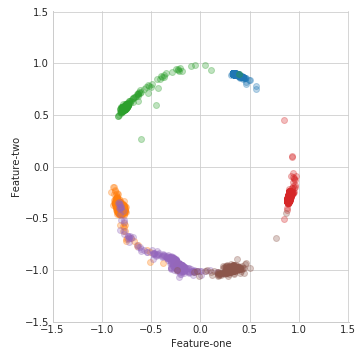} &
		\includegraphics[width=0.1\columnwidth]{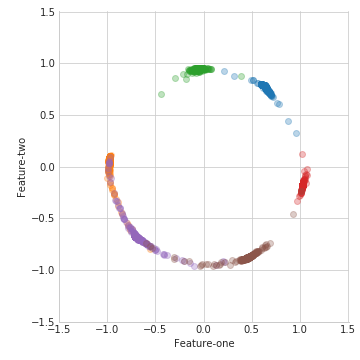} \\
{\scriptsize		(i) Epoch = 0} & {\scriptsize (j)  10}&{\scriptsize (k) 30} & {\scriptsize (l)  500} \\
	\end{tabular}
	\caption{
Evolving learned features at Epoch $=$ 0, 10, 30, and 500. The first, second and third
rows show concentric circles, moons
and 3D Gaussian datasets, respectively.}
	\label{fig:latent}
\end{figure}

\subsection{Real datasets}
{\color{black}We benchmark DDR on a variety of real datasets from both regression and classification problems. Summary information of those datasets used in the analysis is given in Table \ref{tab:datasets}.}
\begin{table}[htbp]
\scriptsize
\centering
\caption{\color{black} Summary information for real datasets.}
\label{tab:datasets}
\begin{tabular}{@{}lccccc@{}}
\toprule
Dataset                & Feature size                     & Training size & Test size&Task \\\midrule
YearPredictionMSD      & $90$                       & $412,276$     & $103,069$& Regression \\
Pole-Telecommunication & $48$                       & $12,000$      & $3,000$  & Regression \\
MNIST                  & $28\times 28\times 1$  & $60k$         & $10k$  & Classification with 10 categories    \\
Kuzushiji-MNIST        & $28\times 28\times 1$  & $60k$         & $10k$    & Classification with 10 categories \\
FashionMNIST           & $28\times 28\times 1$  & $60k$         & $10k$  & Classification with 10 categories   \\
CIFAR-10                & $32\times 32\times 3$  & $50k$         & $10k$  & Classification with 10 categories \\
CIFAR-100                & $32\times 32\times 3$  & $50k$         & $10k$  & Classification with 100 categories  \\ \bottomrule
\end{tabular}
\vskip -0.1in
\end{table}

\noindent \textbf{Regression.}
We benchmark the prediction performance of regression models using the representations learned based on DDR.
Here, we use the  YearPredictionMSD dataset\footnote {The YearPredictionMSD dataset is available at \url{https://archive.ics.uci.edu/ml/datasets/YearPredictionMSD}.} and the
Pole-Telecommunication dataset\footnote {The Pole-Telecommunication dataset is available at \url{https://www.dcc.fc.up.pt/~ltorgo/Regression/DataSets.html}.}.
The YearPredictionMSD dataset contains 515,345 observations with 90 predictors.
The problem is to predict the year of song release.
The Pole-Telecommunication dataset consists of 15,000 observations with 48 predictors for determining the placement of antennas.
We randomly split the data into five folds to evaluate the prediction performance using 5-fold cross validation.
We employ a 3-layer network 
for both $D_{\vphi}$ and $R_{\vtheta}$ on the YearPredictionMSD dataset; a 2-layer network for $D_{\vphi}$ and a 4-layer network $R_{\vtheta}$ are adopted on the Pole-Telecommunication dataset. 
In comparison, we conduct a nonlinear regression using neural networks (NN) with a least squares (LS) loss in the last layer, denoted as NN + LS. That is, we do not impose any desired characteristics for the learned representations in the pultimate layer. Note that, for both DDR and NN+LS, we use the same networks to learn representative features.
We also consider the popular dimension reduction methods, including principal component analysis (PCA) and sparse principal component analysis (SPCA), to obtain data representation.
{\color{black} For the comparison with supervised dimension reduction methods, we
consider SIR and SAVE and the deep learning based sufficient dimension reduction method dCorAE.
For those methods, we first obtain the estimated representative features and fit a linear regression model of the response on the learned representations.
The average prediction errors  and their standard errors based on DDR, NN+LS, dCorAE, PCA, SPCA, SIR, SAVE  and  the ordinary least squares (OLS) regression with the original data  are reported in Tables \ref{tab:music} and \ref{tab:pole}.  DDR outperforms other methods in terms of prediction accuracy.}

\begin{table}[htbp]
	\caption{Prediction error $\pm$ standard error:  YearPredictionMSD dataset}
	\label{tab:music}
	\vskip 0.15in
	\begin{center}
		\scriptsize
			\begin{tabular}{lcccc}
				\toprule
		Methods    & $d = 10$ & $d=20$ & $d=30$ &$d =40$ \\
				\midrule
	DDR &	\textbf{8.8	 \(\pm\) 0.1}	&	\textbf{8.9	 \(\pm\) 0.1}	&	\textbf{8.9	 \(\pm\) 0.1}	&	\textbf{8.8	 \(\pm\) 0.1}	\\
					dCorAE	&	8.9	 \(\pm\) 0.1	&	9.0	 \(\pm\) 0.1	&	9.2	 \(\pm\) 0.1	&	8.9	 \(\pm\) 0.1	\\
				NN+LS	&	9.2	 \(\pm\) 0.1	&	9.3	 \(\pm\) 0.1	&	9.2	 \(\pm\) 0.1	&	9.2	 \(\pm\) 0.1	\\
				SPCA	&	10.6	 \(\pm\) 0.1	&	10.4	 \(\pm\) 0.1	&	9.6	 \(\pm\) 0.1	&	10.2	 \(\pm\) 0.1	\\
				PCA	&		10.6	 \(\pm\) 0.1	&	10.4	 \(\pm\) 0.1	&	10.3	 \(\pm\) 0.1	&	10.2	 \(\pm\) 0.1	\\
				SIR	&	9.6	 \(\pm\) 0.1	&	9.6	 \(\pm\) 0.1	&	9.6	 \(\pm\) 0.1	&	9.6	 \(\pm\) 0.1	\\
				SAVE	&		10.3	 \(\pm\) 0.1	&	9.7	 \(\pm\) 0.1	&	9.6	 \(\pm\) 0.1	&	9.6	 \(\pm\) 0.1	\\
				OLS & \multicolumn{4}{c}{---------9.6	 \(\pm\)0.1---------} \\
				\bottomrule
			\end{tabular}
	\end{center}
	\vskip -0.1in
\end{table}

\begin{table}[htpb]
	\caption{Prediction error $\pm$ standard error:  Pole-Telecommunication dataset}
	\label{tab:pole}
	\vskip 0.15in
	\begin{center}
		\scriptsize
			\begin{tabular}{lcccc}
				\toprule
		Methods    & $d = 5$ & $d=10$ & $d=15$ &$d =20$ \\
				\midrule
				DDR &	 \textbf{2.1   \(\pm\) 0.2} & \textbf{2.1    \(\pm\) 0.1}  & \textbf{2.2    \(\pm\) 0.1} & \textbf{2.2  \(\pm\) 0.2} \\
				dCorAE&  3.1  \(\pm\) 0.1 & 3.1   \(\pm\) 0.3 & 3.1    \(\pm\) 0.2 & 3.0    \(\pm\) 0.1 \\
				NN +LS&  3.0   \(\pm\) 0.5 & 3.1   \(\pm\) 1.1 & 2.7     \(\pm\) 0.8 & 3.2    \(\pm\) 0.6 \\
				SPCA	&	40.3  \(\pm\) 0.3 & 40.1   \(\pm\) 0.3  & 30.5   \(\pm\) 0.2& 30.5  \(\pm\) 0.2 \\
				PCA	&		40.3  \(\pm\) 0.3 & 40.1   \(\pm\) 0.3  & 30.5   \(\pm\) 0.2 & 30.5   \(\pm\) 0.2 \\
				SIR	&	30.4  \(\pm\) 0.1 & 30.4   \(\pm\) 0.1  & 30.5   \(\pm\) 0.1& 30.5  \(\pm\) 0.1 \\
				SAVE	&		31.2  \(\pm\) 0.3 & 30.5   \(\pm\) 0.1  & 30.5   \(\pm\) 0.1& 30.5   \(\pm\) 0.1 \\
				OLS & \multicolumn{4}{c}{---------30.5	 \(\pm\) 0.2---------} \\
				\bottomrule
			\end{tabular}
	\end{center}
	\vskip -0.1in
\end{table}
\noindent
\textbf{\color{black}Classification \uppercase\expandafter{\romannumeral1}.}
{\color{black}We benchmark the classification performance of DDR  using MNIST \citep{mnist}, FashionMNIST \citep{fashion_mnist}, CIFAR-10, and CIFAR-100 \citep{krizhevsky2009learning} datasets against some existing methods, including neural networks (NN) with cross entropy (CN) loss as the last layer, denoted as CNN,  and distance correlation autoencoder (dCorAE) \citep{wang2018distance}.
With CNN, we use  the feature extractor by dropping the last layer for the CN loss of the NN trained for classification as networks. Note that, for both DDR and CNN, we apply the same networks to learn representations.  }

{\color{black}The MNIST and FashionMNIST datasets consist of $60k$ and $10k$ grayscale images with $28\times 28$ pixels for training and testing,
respectively,  while the CIFAR-10 and CIFAR-100 datasets contain $50k$ and $10k$  colored images with $32\times 32$ pixels for training and testing, respectively.}
{\color{black}
The representer network $ R_{\vtheta}$ contains 20 layers for MNIST data and 100 layers for CIFAR-10 data.

{\color{black}To fully utilize computational resources and improve classification accuracy, we further combine DDR with the CN loss, denoted as DDR+CN,  by applying the transfer learning technique \citep{torrey2010transfer, pan2009survey, tan2018survey} on CIFAR-10 and CIFAR-100. 
Data structures for both CIFAR-10 and ImageNet are the same (with three channels), which makes the use of transfer learning straightforward by leveraging the pretrained model of ImageNet.
The pretrained WideResnet-101 model \citep{zagoruyko2016wide} on the ImageNet dataset with Spinal FC \citep{kabir2020spinalnet} is chosen for $R_{\vtheta}$. 
In our experiments for transfer learning, we  first train the WideResnet model on ImageNet. We then use the parameters of the pretrained neural network as the initialization parameters to train CIFAR-10. 
In contrast to transfer learning, the initialization parameters of learning from scratch  are random.}  
The discriminator network $D_{\vphi}$ is a 4-layer network.
The architecture of  $R_{\vtheta}$ and most hyperparameters are shared across all four methods - DDR, {\color{black}CNN}, DDR+CN and dCorAE~\citep{wang2018distance}.}
Finally, we use the $k$-nearest neighbor ($k=5$) classifier on the learned features for all methods.

As shown in Table \ref{tab:mnist},  the classification accuracies of DDR for MNIST and FashionMNIST are better than or comparable with those of {\color{black}CNN} and dCorAE.
{\color{black}As shown in Table \ref{tab:cifar10}, the classification accuracy of DDR using the CN loss outperforms that of CNN on CIFAR-10 and CIFAR-100.}
We also calculate the estimated distance correlation (DC)
between the learned features and their labels.
Figure \ref{fig:dc} shows the values of DC for MNIST, FashionMNIST and CIFAR-10 data.
Higher DC values mean that the learned features are of higher quality.
DDR and DDR+CN achieves higher DC values.

\begin{figure}[htbp]
	\centering
	\begin{tabular}{cccccc}
		\includegraphics[width=0.14\columnwidth]{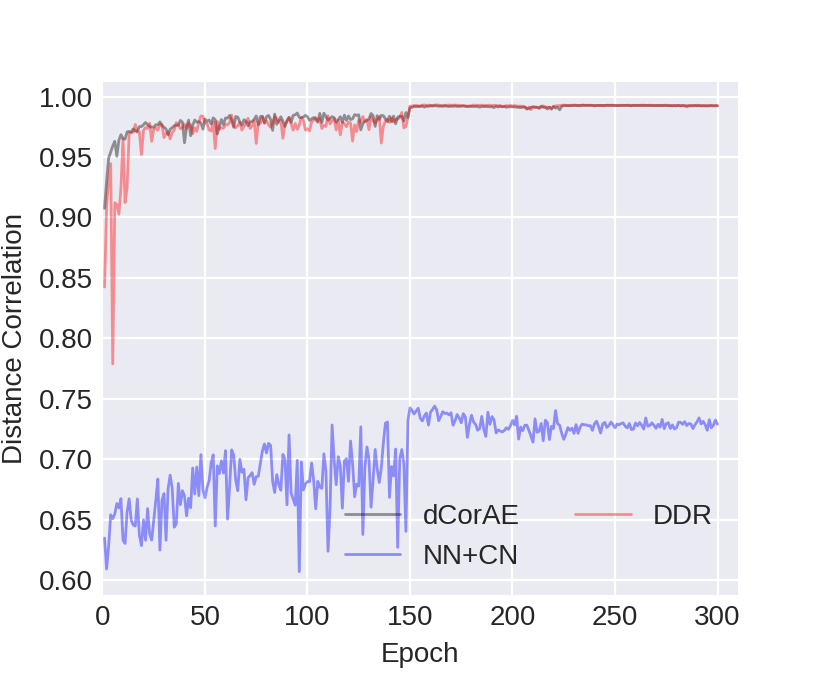}&
		\includegraphics[width=0.14\columnwidth]{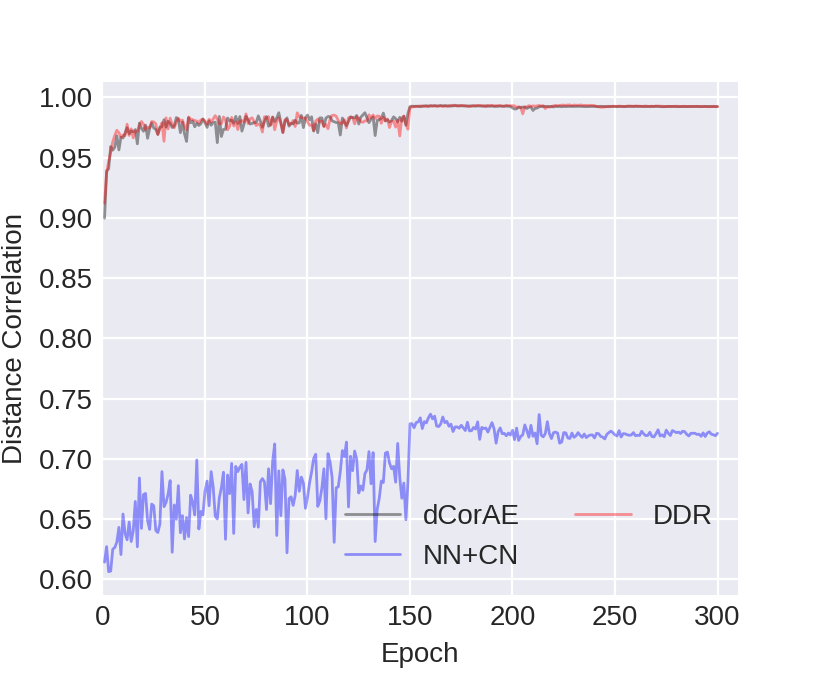} &
		\includegraphics[width=0.14\columnwidth]{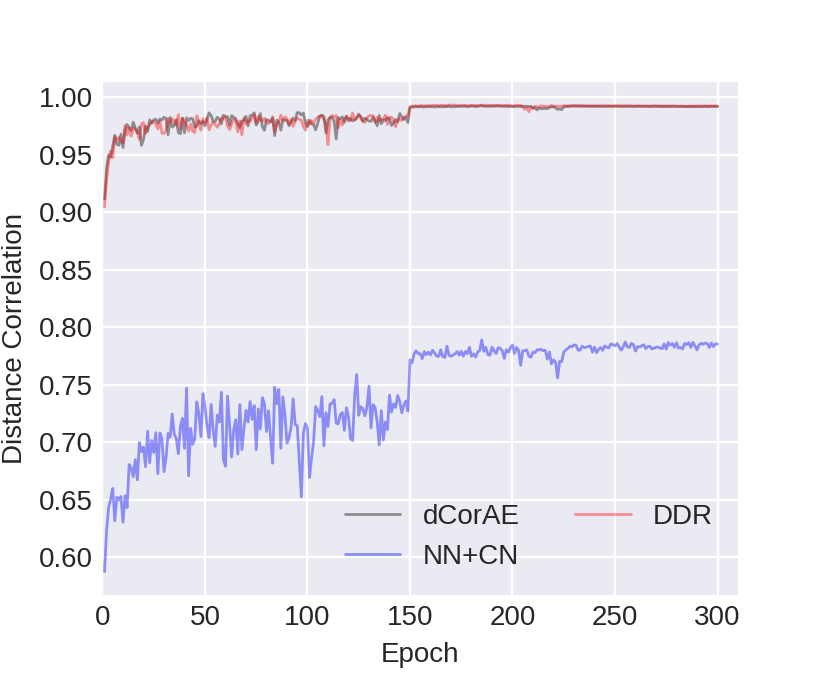} &
	\includegraphics[width=0.14\columnwidth]{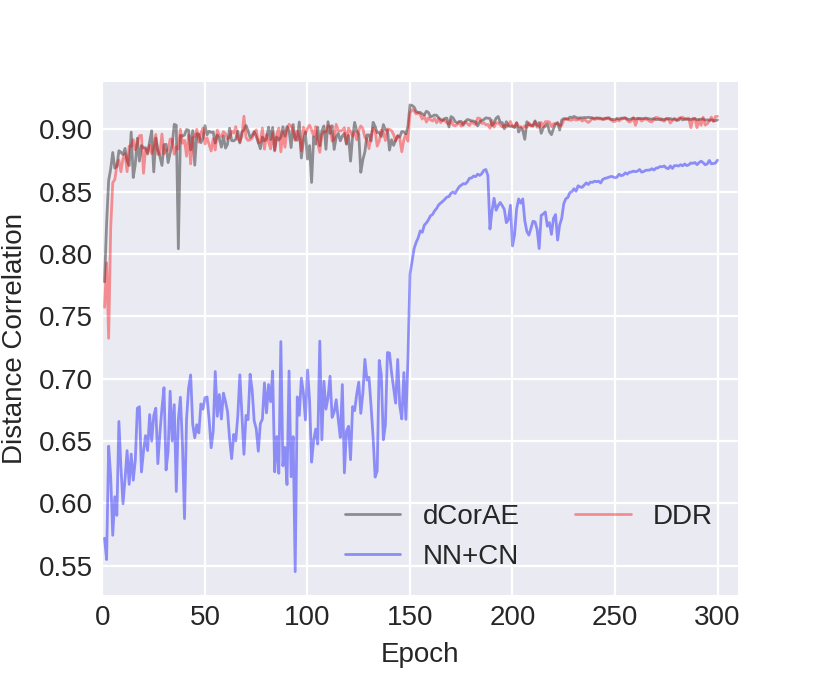}&
		\includegraphics[width=0.14\columnwidth]{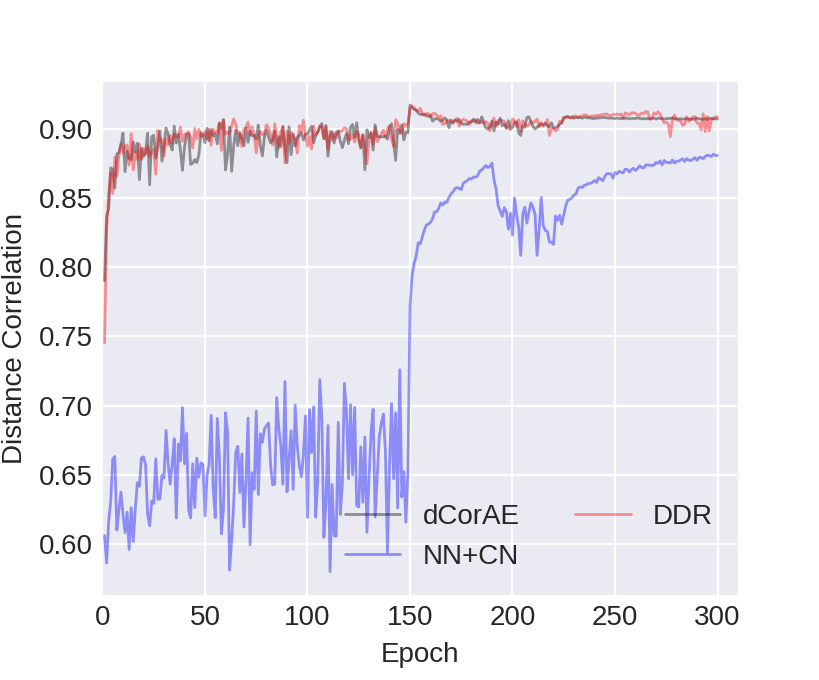} &
		\includegraphics[width=0.14\columnwidth]{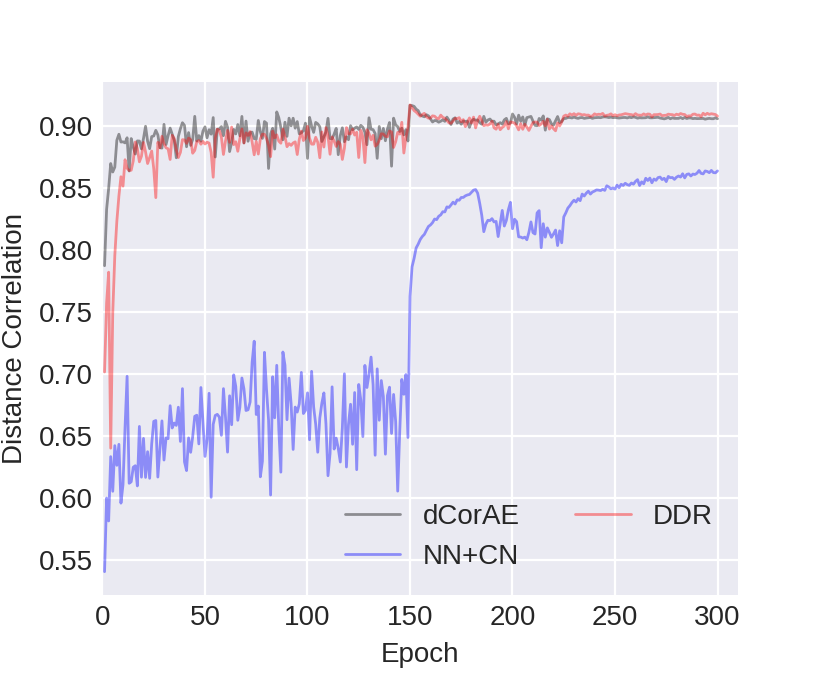}\\
{\tiny	(a) MNIST, $16$} & {\tiny (b) MNIST, $32$} & {\tiny (c) MNIST, $64$ }&
{\tiny	(d) FMNIST, $16$} &  {\tiny (e) FMNIST, $32$}
& {\tiny (f)  FMNIST, $64 $}
\end{tabular}

\begin{tabular}{cccccc}
		\includegraphics[width=0.14\columnwidth]{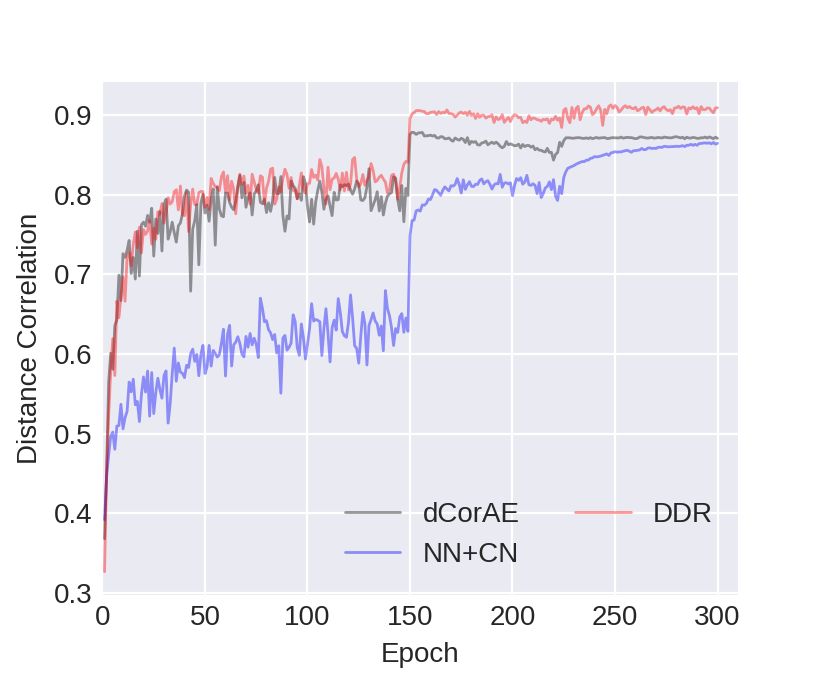}&
		\includegraphics[width=0.14\columnwidth]{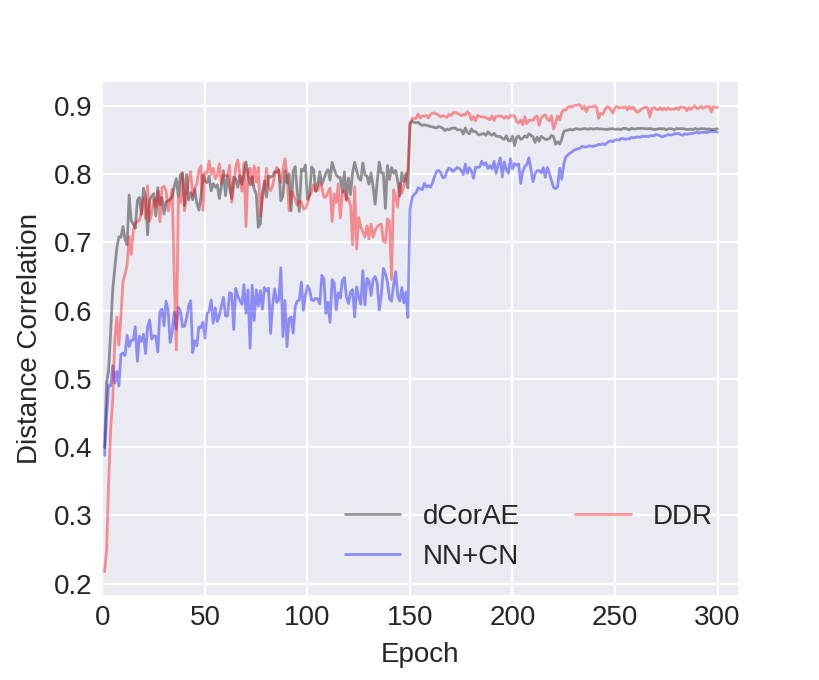} &
		\includegraphics[width=0.14\columnwidth]{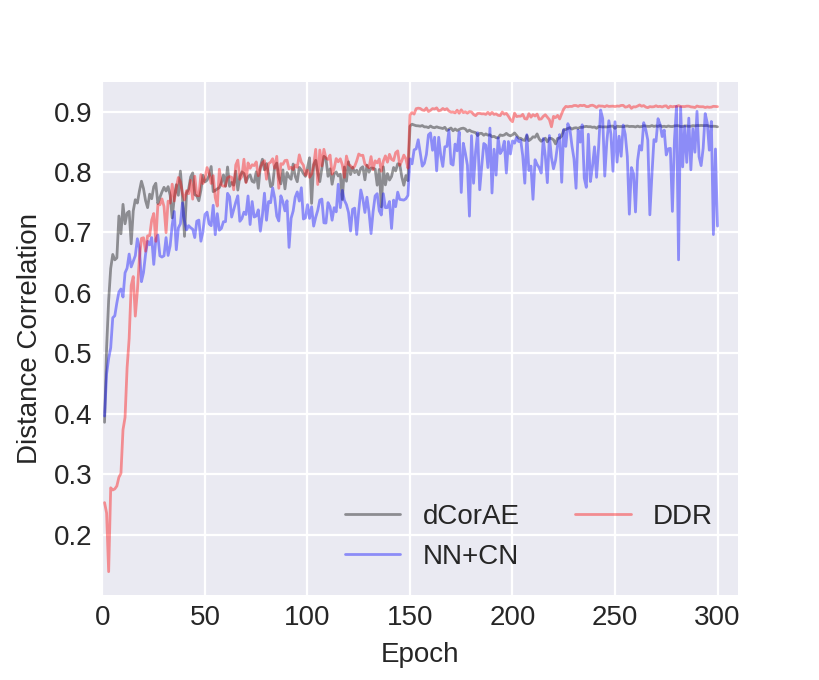} &
		\includegraphics[width=0.14\columnwidth]{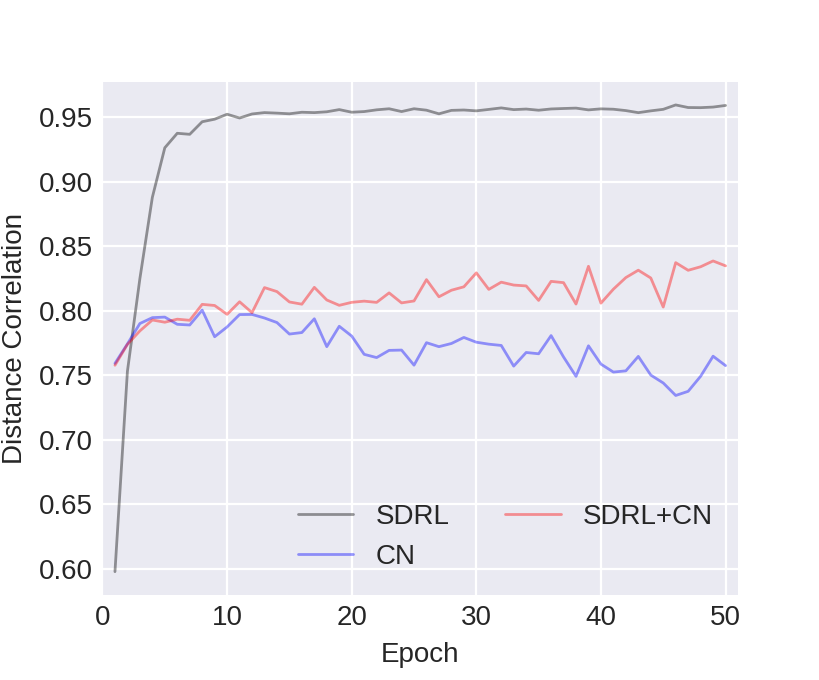}&
		\includegraphics[width=0.14\columnwidth]{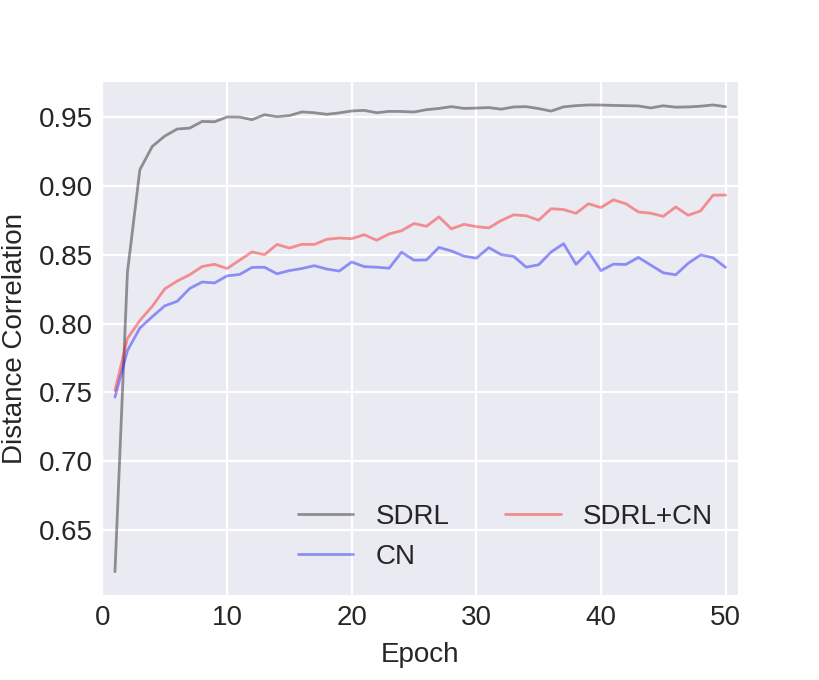} &
		\includegraphics[width=0.14\columnwidth]{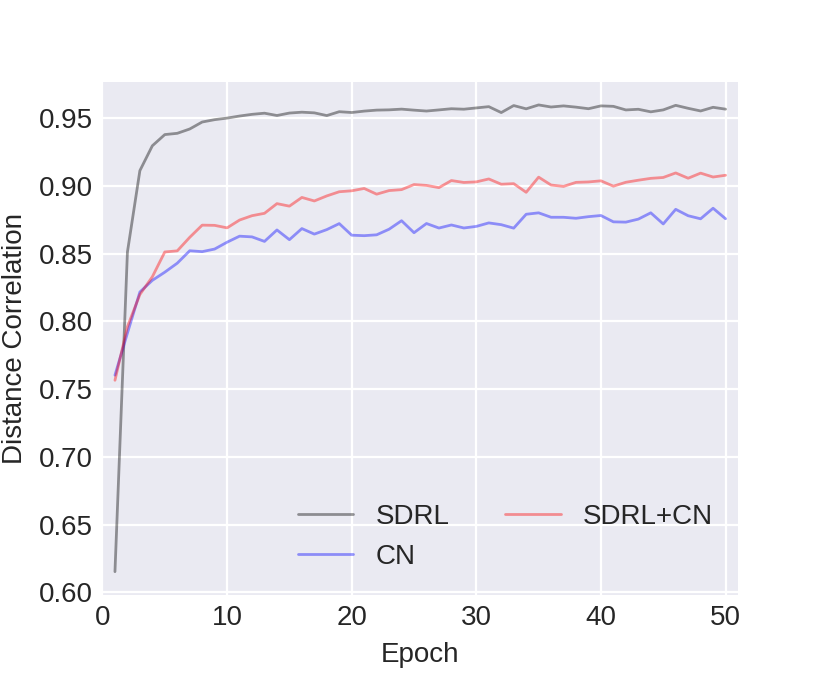}\\
{\tiny (g) CIFAR (FS), $16$} &{\tiny (h) CIFAR (FS), $32$}&
 {\tiny (i) CIFAR (FS), $64$ } &
{\tiny		(j) CIFAR (TL), $16$} &{\tiny (k) CIFAR (TL), $32$}&
 {\tiny (l) CIFAR (TL), $64$ }
	\end{tabular}
	\caption{\color{black}The distance correlations of labels with learned features based on DDR, CNN and  dCorAE with {\color{black} $d=16, 32$ and 64} for MNIST,
FashionMNIST (FMNIST)  and CIFAR-10 (CIFAR) data (FS: from scratch; TL: transfer learning).}
	\label{fig:dc}
\end{figure}



{ \color{black} Because both GSIR and GSAVE require computation of $n\times n$ kernel matrices, which is computationally prohibitive when $n=10,000$ to $60,000$ and $p \approx 1,000$, it does not allow us to apply both methods to analyze datasets from MNIST, FashionMNIST, CIFAR-10 and CIFAR-100}.
\begin{table}[ht!]
	\caption{{\color{black} Classification accuracy for MNIST and FashionMNIST. In the table, dCor=dCorAE.}}
	\label{tab:mnist}
	\begin{center}
		\scriptsize
			\begin{tabular}{lccccccc}
				\toprule
				&&\multicolumn{1}{c}{ MNIST} & &\multicolumn{3}{c}{ FashionMNIST}  \\
				\cmidrule(r){2-4} \cmidrule(r){5-7}
				$d$	   		&	DDR	&	dCor	& CNN	&	DDR	&	dCor	&	CNN	\\
				\midrule
				$d=16$		&	99.41	&	99.58	&	99.39	&	\textbf{94.44}	&	94.18	&	94.21		\\
				$d=32$		&	\textbf{99.61}	&	99.54	&	99.45	&	94.18	&	93.89	&	94.41	\\
				$d=64$		&	\textbf{99.56}	&	99.53	&	99.49	&	94.13	&	94.24	&	94.38	\\
				\bottomrule
			\end{tabular}
	\end{center}
	\vskip -0.1in
\end{table}

\begin{table}[ht!]
	\caption{{\color{black} Classification accuracy for CIFAR-10 and CIFAR-100 data. 
 }}
	\label{tab:cifar10}
	\begin{center}
		\scriptsize
			\begin{tabular}{@{}lcccccclccc@{}}
\toprule
       & \multicolumn{6}{c}{CIFAR-10}                                                                       && \multicolumn{3}{c}{CIFAR-100}                              \\ 	\cmidrule(r){2-7} \cmidrule(r){8-11}
       & \multicolumn{3}{c}{Learning from scratch}        & \multicolumn{3}{c}{Transfer learning}           && \multicolumn{3}{c}{Transfer learning}                      \\
				\cmidrule(r){2-4} \cmidrule(r){5-7} \cmidrule(r){9-11}
$d$    & dCorAE & CNN   & DDR                             & CNN   & DDR   & DDR+CN                          & $d$     & dCorAE & CNN   & DDR+CN                          \\\midrule
$d=16$ & 94.15  & 94.21 & \textbf{94.29} & 97.44 & 97.52 & \textbf{97.68} & $d=200$ & 85.39  & 86.29 & \textbf{86.36} \\
$d=32$ & 94.18  & 94.92 & 94.58                           & 97.79 & 97.33 & \textbf{97.96} & $d=300$ & 85.57  & 85.95 & \textbf{86.04} \\
$d=64$ & 94.66  & 95.09 & 94.46                           & 97.90 & 97.49 & \textbf{97.91} & $d=400$ & 85.55  & 86.21 & \textbf{86.30} \\ \bottomrule
\end{tabular}
	\end{center}
	\vskip -0.1in
\end{table}

\noindent
{\color{black} \textbf{Classification \uppercase\expandafter{\romannumeral2}.}  To compare the performance of DDR with semi-supervised methods, we benchmark DDR on MNIST dataset with varying amounts of labeled data for training.  In detail, we consider some widely used 
semi-supervised learning methods, including semi-supervised variational autoencoders (Semi-VAE) \citep{kingma2014semi} and information maximizing variational autoencoders (InfoVAE) with the semi-supervised setting \citep{kingma2014semi, zhao2019infovae}, and the supervised method, CNN. 
InfoVAE utilizes all training images to learn representations, and then trains the $k$-nearest neighbor ($k=5$) classifier with the learned representations and partially known labels. 
All four methods share the same network architecture for $50$-dimensional learning representation.
We adopt the double-hidden-layer MLP networks, with 600 neurons for each layer, the softplus activation function, and the Adam optimizer.
For both Semi-VAE and InfoVAE, we apply a semi-supervised setting to analyze all $60k$ images with the varying number of labeled images as the training data and validate the performance using $10k$ test data.
But for both DDR and CNN, we apply a supervised setting to analyze only data with labels and discard the rest of images in the training set.
\begin{figure}[htbp]
\begin{center}
\includegraphics[width=0.9\textwidth]{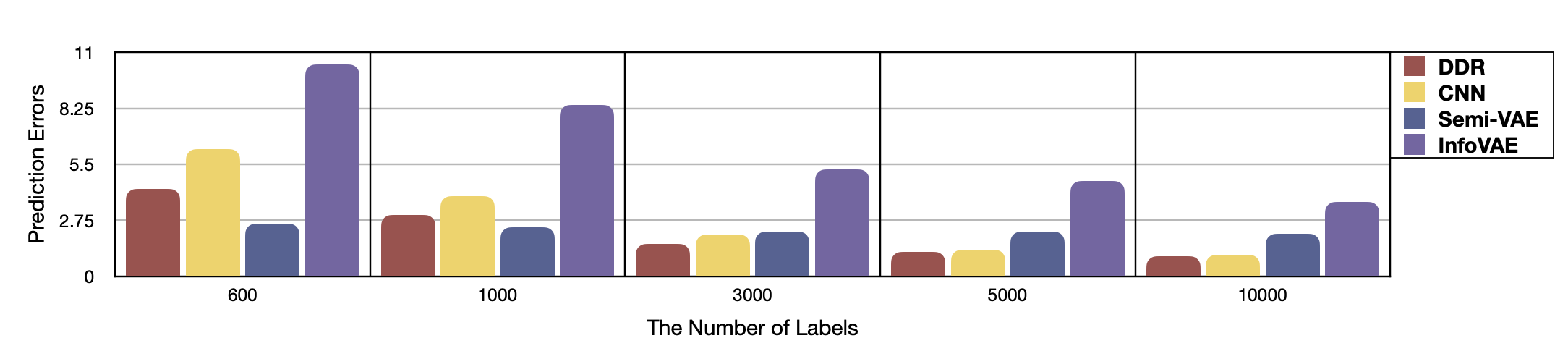}
\caption{\color{black} The classification errors comparison on varying amounts of labeled MNIST data. }
\label{fig:VAE}
\end{center}
\end{figure}

The prediction accuracy using images with the varying number of labels, 
from $600$ to $10,000$, is shown in Fig. \ref{fig:VAE}. For supervised methods, we observe that the accuracy of DDR outperforms that of CNN for the varying number of images with labels. When the proportion of images with labels is low,  Semi-VAE performs the best 
by fully utilizing all $60k$ training images.
However, the accuracy of Semi-VAE does not improve over the increasing proportion of images with labels.   
This is because the objective function of Semi-VAE is primary to maximize the lower bound of the joint likelihood rather than the classification loss. 
In all, compared with semi-supervised learning, DDR uses a small amount of data with labels in training, but achieves better classification accuracy when the number of images with labels are larger than 1,000.
}
\section{Conclusion and future work}
\label{conclusion}
{\color{black}
Since the framework of supervised dimension reduction was first introduced over twenty years ago by \citet{li1991sliced}, there have been a variety of important methods developed in this framework.
However, most of the existing works focused on finding linear representations of the input data. There are two important reasons why it has been difficult to develop nonparametric supervised representation learning approaches.
First, supervised representation learning is more difficult than and fundamentally different
from supervised learning.  Indeed, it is challenging to formulate a clear and simple objective function for supervised representation learning in the first place. This is in clear contrast to supervised learning,
whose objective is clear-cut.  For example, in classification, the objective is to minimize the misclassification rate or surrogate objective function; in regression, a least squares criterion for the fitting error is usually used. However,  how to construct an objective function for nonparametric supervised representation
learning in a principled way has remained an open question \citep{bengio2013representation, alain2017understanding}.
Second, it is difficult to apply standard techniques for nonparametric estimation such as smoothing,  splines and kernel methods in multi-dimensional problems. They are either not flexible enough for providing accurate adaptive and data-driven based approximation to multi-dimensional functions or are computational prohibitive with high-dimensional data.
}

In this work, we propose a nonparametric DDR approach to achieving a good data representation for supervised learning with certain desired characteristics including sufficiency, low-dimensionality and disentanglement.
We estimate the representation map nonparametrically by taking advantage of the powerful capabilities of deep neural networks in approximating multi-dimensional functions.
The proposed DDR is validated via comprehensive numerical experiments and real data analysis in the context of regression and classification.

Several questions deserve further study.
First,  it would be interesting to consider other measures of conditional independence such as
conditional covariance operators on reproducing kernel Hilbert spaces
\citep{fukumizu2009kernel} and heteroscedastic conditional variance operator on Hilbert spaces \citep{lee2013gen}.
It is also possible to use mutual information for measuring conditional independence \citep{suzuki2013sufficient}, although with this measure the loss function itself needs to be estimated.
It would be interesting to develop algorithms and theoretical understanding of these criteria and evaluate the relative performance of the learned representations based on
these different conditional independence measures.

We used the standard Gaussian as the reference distribution for DDRM to promote disentanglement of the representation. Another convenient choice is the uniform distribution on the unit cube.
It is worth examining whether there is any difference in the performance of DDR with different reference distributions. Another question is how to determine the dimension of the learned representation. This is also an important problem in linear SDR. Since the purpose of dimension reduction is often to build prediction models in high-dimensional settings, the problem of determining the dimension of the representation can be best addressed based on cross validation or related data-driven methods in the model building phase.

{\color{black} Last but not least, due to the non-uniqueness of the target, it is challenging to provide the consistency of the estimated nonlinear dimension reduction map. It will be interesting to explore this property in the future work.  } 
\newpage
\section{Appendix}
{\color{black}In the appendix, we show additional experiments using DDR and InfoVAE.
In addition, we provide the implementation details about numerical settings, network structures, SGD optimizers, and hyper-parameters used in the numerical studies. }
We also give the detailed proofs of Lemmas \ref{lem1}-\ref{glmin} and Theorem \ref{esterr}.
\subsection{Additional experiments}
{\color{black}We make comparisons of DDR and InfoVAE on MNIST, Kuzushiji-MNIST and Fashion-MNIST.
Because InfoVAE is an unsupervised method based on autoencoders with maximum mean discrepancy (MMD) loss for constrained representation, we use the semi-supervised setting of InfoVAE \citep{kingma2014semi, zhao2019infovae} in our experiments. 
We utilize the double-hidden-layer MLP networks with the softplus activation function, and the Adam optimizer.
\begin{table}[htbp]
\caption{\color{black}MMD metric and classification error for MNIST,  Kuzushiji-MNIST and Fashion-MNIST.}
	\label{tab:info1}
	\begin{center}
		\scriptsize
\begin{tabular}{@{}ll|cccc|cccc@{}}
\toprule
                                                      &                          & \multicolumn{4}{c|}{MMD (\%)}                       & \multicolumn{4}{c}{Classification error (\%)}                          \\ \midrule
\multicolumn{1}{l|}{}                                 & \multicolumn{1}{c|}{\#label} & 600             & 1000            & 3000   & 5000   & 600            & 1000           & 3000           & 5000           \\ \midrule
\multicolumn{1}{l|}{\multirow{2}{*}{MNIST}}           & DDR                      & \textbf{0.0249} & \textbf{0.0242} & 0.0367 & 0.127  & \textbf{9.92}  & \textbf{6.70}  & \textbf{2.80}  & \textbf{2.03}  \\
\multicolumn{1}{l|}{}                                 & InfoVAE                & 0.0272          & 0.0273          & 0.0273 & 0.0273 & 13.86          & 11.27          & 7.42           & 6.29           \\ \midrule
\multicolumn{1}{l|}{\multirow{2}{*}{Kuzushiji-MNIST}} & DDR                      & 0.0214          & 0.0242          & 0.0287 & 0.0247 & \textbf{27.24} & \textbf{22.75} & \textbf{14.04} & \textbf{9.63}  \\
\multicolumn{1}{l|}{}                                 &InfoVAE                 & 0.0202          & 0.0202          & 0.0202 & 0.0202 & 34.58          & 28.57          & 18.89          & 15.19          \\ \midrule
\multicolumn{1}{l|}{\multirow{2}{*}{Fashion-MNIST}}   & DDR                      & \textbf{0.0278} & \textbf{0.0344} & 0.0475 & 0.0474 & \textbf{21.25} & \textbf{18.36} & \textbf{16.40} & \textbf{15.15} \\
\multicolumn{1}{l|}{}                                 & InfoVAE                & 0.0404          & 0.0404          & 0.0404 & 0.0404 & 21.47          & 20.27          & 17.66          & 16.67          \\ \bottomrule
\end{tabular}
\end{center}
\vskip -0.1in
\end{table}

%
%
%
We compare  DDR with InfoVAE in terms of classification error and MMD metric, as shown in Table \ref{tab:info1}.
We observe that DDR and InfoVAE are comparable to learn a standard Gaussian distribution, but DDR outperforms InfoVAE across all settings in terms of the classification accuracy.
In DDR, there exists a trade-off between two objectives: learn a Gaussian distribution and improve the performance of downstream tasks.
From the experimental results, we conclude that DDR can achieve satisfactory prediction accuracy with mild constraints on the representation distributions.}
\subsection{Experimental details}

\subsubsection{Simulation studies}
The values of the hyper-parameters for the simulated experiments are given in Table \ref{toy}, where \(\lambda\) is the penalty parameter, \(d \) is the dimension of the SDRM, \(n\) is the mini-batch size in SGD,  \(T_1\) is the number of inner loops to push forward particles \(\rvz_i \), \(T_2\) is the number of outer loops for training \(R_{\vtheta}\), and \(s\) is the step size to update particles. For the regression models, the neural network architectures are shown in {\color{black}Table \ref{mlpr1}.}

As shown in Table \ref{mlp}, a multilayer perceptron (MLP) is utilized for the neural structure \(D_{\vphi}\) in the classification  problem. The detailed architecture of 10-layer dense convolutional network (DenseNet) \citep{huang2017denselyb, amos2017densenet} deployed for \(R_{\vtheta}\) is shown in Table \ref{toy_R}. For all the settings, we adopted the Adam \citep{kingma2014adam} optimizer with an initial learning rate of 0.001 and weight decay of 0.0001.
\begin{table}[htbp]
	\caption{\label{toy}Hyper-parameters for simulated examples, where \(s\) varies according to epoch}
		\begin{center}
\footnotesize
			\begin{tabular}{lcccccccc}
				\toprule
				&& & &&&\multicolumn{3}{c}{\(s\)} \\
				\cmidrule(r){7-9}
				Task& \(\lambda\) & \(d\) & \(n\)& \(T_1\) & \(T_2\)& 0-150  &151-225 & 226-500  \\
				\midrule
				Regression  
&1.0&2 or 1&64& 1&500&3.0&2.0&1.0\\
				Classification&1.0&2&64&1&500&2.0&1.5&1.0\\
				\bottomrule
			\end{tabular}
		\end{center}
		\vskip -0.1in
\end{table}

\begin{table}[htbp]
	\caption{\label{mlpr1}MLP architectures for  \(D_{\vphi}\) and \(R_{\vtheta}\) in regression } 
		\begin{center}
\footnotesize
			\begin{tabular}{lccccc}
				\toprule
				&\multicolumn{2}{c} { \(D_{\vphi}\) } &&\multicolumn{2}{c}{\(R_{\vtheta}\)} \\
				\cmidrule(r){2-3}  \cmidrule(r){5-6}
				Layers & Details & Output size&& Details & Output size \\
				\midrule
				Layer 1 & Linear, LeakyReLU& 16&& Linear, LeakyReLU& 16 \\
				Layer 2 & Linear& 1&& Linear, LeakyReLU& 8 \\
				Layer 3 & & && Linear&   \(d\) \\
				\bottomrule
			\end{tabular}
		\end{center}
		\vskip -0.1in
\end{table}

\begin{table}[htbp]
	\caption{\label{mlp} MLP architecture for  \(D_{\vphi}\) in the simulated classification examples and  the benchmark classification datasets}
		\begin{center}
\footnotesize
			\begin{tabular}{lcc}
				\toprule
				Layers & Details & Output size \\
				\midrule
				Layer 1 & Linear, LeakyReLU& 64 \\
				Layer 2 & Linear, LeakyReLU& 128 \\
				Layer 3 & Linear, LeakyReLU& 64 \\
				Layer 4 & Linear& 1 \\
				\bottomrule
			\end{tabular}
		\end{center}
		\vskip -0.1in
\end{table}

\begin{table}[htbp]
	\caption{\label{toy_R}DenseNet architecture for  \(R_{\vtheta}\) in the simulated classification examples}
		\begin{center}
\footnotesize
			\begin{tabular}{lcc}
				\toprule
				Layers & Details & Output size \\
				\midrule
				Convolution & \(3 \times 3\) Conv& \(24 \times 20\times 20\) \\
				Dense Block 1& \(\left[\begin{array}{l}\text{BN, } 1 \times 1 \text { Conv } \\ \text{BN, } 3 \times 3 \text { Conv }\end{array}\right] \times 1\)& \(36 \times 20 \times 20\) \\
				Transition Layer 1 &  BN, ReLU, \(2 \times 2\)  Average Pool,\(1 \times 1\) Conv &\(30 \times10 \times 10\) \\
				Dense Block 2& \(\left[\begin{array}{l}\text{BN, } 1 \times 1 \text { Conv } \\ \text{BN, } 3 \times 3 \text { Conv }\end{array}\right] \times 1\)& \(18 \times 10\times 10\) \\
				Transition Layer 2&  BN, ReLU, \(2 \times 2\)  Average Pool, \(1 \times 1\) Conv &\(15 \times5 \times 5\) \\
				Dense Block 3& \(\left[\begin{array}{l}\text{BN, } 1 \times 1 \text { Conv } \\ \text{BN, } 3 \times 3 \text { Conv }\end{array}\right] \times 1\)& \(27 \times 5 \times 5\) \\
				Pooling&   BN, ReLU, \(5 \times 5\)  Average Pool, Reshape&\(27\) \\
				Fully connected&  Linear&\(2\) \\
				\bottomrule
			\end{tabular}
		\end{center}
		\vskip -0.1in
\end{table}

\subsubsection{Real datasets}
\textbf{Regression:}  In the regression problems,  hyper-parameters
are presented in Table \ref{reg}.
The Adam optimizer with an initial learning rate of 0.001 and weight decay of 0.0001 is adopted. The MLP architectures of  \(D_{\vphi}\) and \(R_{\vtheta}\)
for the YearPredictionMSD data are shown in Table \ref{mlp2} {\color{black} and for the Pole-Telecommunication data are shown in Table \ref{mlp4}.}

\begin{table}[htbp]
	\caption{\label{reg}{\color{black}Hyper-parameters for real regression datasets}}
		\begin{center}
\footnotesize
			\begin{tabular}{lcccccc}
				\toprule
				Dataset& \(\lambda\) & \(d\) & \(n\)& \(T_1\) & \(T_2\)& \(s\)  \\
				\midrule
				YearPredictionMSD&1.0&10, 20, 30, 40&64& 1&500&1.0\\
				Pole-Telecommunication&1.0&5, 10, 15, 20&64&1&200&1.0\\
				\bottomrule
			\end{tabular}
		\end{center}
		\vskip -0.1in
\end{table}
\begin{table}[htbp]
	\caption{\label{mlp2}MLP architectures for  \(D_{\vphi}\) and \(R_{\vtheta}\) for YearPredictionMSD data}
		\begin{center}
\footnotesize
			\begin{tabular}{lccccc}
				\toprule
				&\multicolumn{2}{c} { \(D_{\vphi}\)}&&\multicolumn{2}{c}{\(R_{\vtheta}\)} \\
				\cmidrule(r){2-3}  \cmidrule(r){5-6}
				Layers & Details & Output size&& Details & Output size \\
				\midrule
				Layer 1 & Linear, LeakyReLU& 32&& Linear, LeakyReLU& 32 \\
				Layer 2 & Linear, LeakyReLU& 8&& Linear, LeakyReLU& 8 \\
				Layer 3 & Linear& 1&& Linear&   \(d\) \\
				\bottomrule
			\end{tabular}
		\end{center}
		\vskip -0.1in
\end{table}

\begin{table}[htbp]
	\caption{\label{mlp4}{\color{black}MLP architectures for  \(D_{\vphi}\) and \(R_{\vtheta}\) for Pole-Telecommunication data}}
		\begin{center}
\footnotesize
			\begin{tabular}{lccccc}
				\toprule
				&\multicolumn{2}{c} { \(D_{\vphi}\)}&&\multicolumn{2}{c}{\(R_{\vtheta}\)} \\
				\cmidrule(r){2-3}  \cmidrule(r){5-6}
				Layers & Details & Output size&& Details & Output size \\
				\midrule
				Layer 1 & Linear, LeakyReLU& 8&& Linear, LeakyReLU& 16 \\
				Layer 2 & Linear& d&& Linear, LeakyReLU& 32 \\
				Layer 3 & & &&Linear, LeakyReLU& 8 \\
				Layer 4 & & && Linear&   \(d\) \\
				\bottomrule
			\end{tabular}
		\end{center}
		\vskip -0.1in
\end{table}

\noindent
\textbf{Classification:}
We again use Adam as the SGD optimizers for both $D_\vphi$ and $R_\vtheta$. Specifically, learning rate of 0.001 and weight decay of 0.0001 are used for $D_\vphi$ in all datasets and for $R_\vtheta$ on MNIST \citep{mnist}. We customized the SGD optimizers with momentum at 0.9, weight decay at 0.0001, and learning rate \(\rho\) in Table \ref{lr}.
{\color{black} For the transfer learning of CIFAR-10, we use customized SGD optimizer with initial learning rate of 0.001 and momentum of 0.9 for $R_\vtheta$.
For}  FashionMNIST \citep{fashion_mnist} and CIFAR-10 \citep{krizhevsky2012imagenet},
MLP architectures 
of the discriminator network \(D_{\vphi}\) for MNIST, FashionMNIST and CIFAR-10 are given in Table \ref{mlp}. The 20-layer DenseNet networks shown in Table \ref{mnist} were utlized for \(R_{\vtheta}\) on the MNIST dataset, while the 100-layer DenseNet networks shown in Table \ref{fmnist} and \ref{cifar10} are fitted for \(R_{\vtheta}\) on FashionMNIST and CIFAR-10.
\begin{table}[ht!]
	\caption{\label{class}{\color{black} Hyper-parameters for the classification benchmark datasets}}
		\begin{center}
\footnotesize
			\begin{tabular}{lcccccc}
				\toprule
				Dataset& \(\lambda\) & \(d\) & \(n\)& \(T_1\) & \(T_2\)& \(s\)  \\
				\midrule
				MNIST&1.0&16, 32, 64&64& 1&300&0.1\\
				FashionMNIST&1.0&16, 32, 64&64&1&300&1.0\\
				CIFAR-10&1.0&16, 32, 64&64&1&300&1.0\\
				CIFAR-10 (transfer learning)&0.01&16, 32, 64&64&1&50&1.0\\
				\bottomrule
			\end{tabular}
		\end{center}
		\vskip -0.1in
\end{table}
\begin{table}[htbp]
	\caption{\label{lr}Learning rate $\rho$ varies  during training.}
		\begin{center}
\footnotesize
			\begin{tabular}{cccccc}
				\toprule
				Epoch&0-150&151-225&226-300\\
				\midrule
				\(\rho\)&0.1&0.01&0.001&\\
				\bottomrule
			\end{tabular}
		\end{center}
		\vskip -0.1in
\end{table}
\begin{table}[htbp]
	\caption{\label{mnist}Architecture for MNIST, reduced feature size is \(d\)}
		\begin{center}
\footnotesize
			\begin{tabular}{lcc}
				\toprule
				Layers & Details & Output size \\
				\midrule
				Convolution & \(3 \times 3\) Conv& \(24 \times 28\times 28\) \\
				Dense Block 1& \(\left[\begin{array}{l}\text{BN, } 1 \times 1 \text { Conv } \\ \text{BN, } 3 \times 3 \text { Conv }\end{array}\right] \times 2\)& \(48 \times 28 \times 28\) \\
				Transition Layer 1 &  BN, ReLU, \(2 \times 2\)  Average Pool,\(1 \times 1\) Conv &\(24 \times14 \times 14\) \\
				Dense Block 2& \(\left[\begin{array}{l}\text{BN, } 1 \times 1 \text { Conv } \\ \text{BN, } 3 \times 3 \text { Conv }\end{array}\right] \times 2\)& \(48 \times 14\times 14\) \\
				Transition Layer 2&  BN, ReLU, \(2 \times 2\)  Average Pool, \(1 \times 1\) Conv &\(24 \times7 \times 7\) \\
				Dense Block 3& \(\left[\begin{array}{l}\text{BN, } 1 \times 1 \text { Conv } \\ \text{BN, } 3 \times 3 \text { Conv }\end{array}\right] \times 2\)& \(48 \times 7 \times 7\) \\
				Pooling&   BN, ReLU, \(7 \times 7\)  Average Pool, Reshape&\(48\) \\
				Fully connected&  Linear&\(d\) \\
				\bottomrule
			\end{tabular}
		\end{center}
		\vskip -0.1in
\end{table}
\begin{table}[htbp]
	\caption{\label{fmnist}Architecture for FashionMNIST, reduced feature size is \(d\)}
		\begin{center}
\footnotesize
			\begin{tabular}{lcc}
				\toprule
				Layers & Details & Output size \\
				\midrule
				Convolution & \(3 \times 3\) Conv& \(24 \times 28\times 28\) \\
				Dense Block 1& \(\left[\begin{array}{l}\text{BN, } 1 \times 1 \text { Conv } \\ \text{BN, } 3 \times 3 \text { Conv }\end{array}\right] \times 16\)& \(216 \times 28 \times 28\) \\
				Transition Layer 1 &  BN, ReLU, \(2 \times 2\)  Average Pool,\(1 \times 1\) Conv &\(108 \times14 \times 14\) \\
				Dense Block 2& \(\left[\begin{array}{l}\text{BN, } 1 \times 1 \text { Conv } \\ \text{BN, } 3 \times 3 \text { Conv }\end{array}\right] \times 16\)& \(300 \times 14\times 14\) \\
				Transition Layer 2&  BN, ReLU, \(2 \times 2\)  Average Pool, \(1 \times 1\) Conv &\(150 \times7 \times 7\) \\
				Dense Block 3& \(\left[\begin{array}{l}\text{BN, } 1 \times 1 \text { Conv } \\ \text{BN, } 3 \times 3 \text { Conv }\end{array}\right] \times 16\)& \(342 \times 7 \times 7\) \\
				Pooling&   BN, ReLU, \(7 \times 7\)  Average Pool, Reshape&\(342\) \\
				Fully connected&  Linear&\(d\)\\
				\bottomrule
			\end{tabular}
		\end{center}
		\vskip -0.1in
\end{table}
\begin{table}[htbp]
	\caption{\label{cifar10}Architecture for CIFAR-10, reduced feature size is \(d\)}
		\begin{center}
\footnotesize
			\begin{tabular}{lcc}
				\toprule
				Layers & Details & Output size \\
				\midrule
				Convolution & \(3 \times 3\) Conv& \(24 \times 32\times 32\) \\
				Dense Block 1& \(\left[\begin{array}{l}\text{BN, } 1 \times 1 \text { Conv } \\ \text{BN, } 3 \times 3 \text { Conv }\end{array}\right] \times 16\)& \(216 \times 32 \times 32\) \\
				Transition Layer 1 &  BN, ReLU, \(2 \times 2\)  Average Pool,\(1 \times 1\) Conv &\(108 \times16 \times 16\) \\
				Dense Block 2& \(\left[\begin{array}{l}\text{BN, } 1 \times 1 \text { Conv } \\ \text{BN, } 3 \times 3 \text { Conv }\end{array}\right] \times 16\)& \(300 \times 16 \times 16\) \\
				Transition Layer 2&  BN, ReLU, \(2 \times 2\)  Average Pool, \(1 \times 1\) Conv &\(150 \times8 \times 8\) \\
				Dense Block 3& \(\left[\begin{array}{l}\text{BN, } 1 \times 1 \text { Conv } \\ \text{BN, } 3 \times 3 \text { Conv }\end{array}\right] \times 16\)& \(342 \times 8 \times 8\) \\
				Pooling&   BN, ReLU, \(8 \times 8\)  Average Pool, Reshape&\(342\) \\
				Fully connected&  Linear&\(d\) \\
				\bottomrule
			\end{tabular}
		\end{center}
		\vskip -0.1in
\end{table}

\newpage
\subsection{Proofs }
In this section, we prove Lemmas \ref{lem1} and \ref{lem2}, and Theorems \ref{glmin} and \ref{esterr}.

\subsubsection{Proof of Lemma \ref{lem1}}
\begin{proof}
	By assumption  $\mu$ and $\gamma_{d}$ are both absolutely continuous with respect to  the  Lebesgue measure. The desired result holds since it  is a special case of the  well known results on the existence of optimal transport  \citep{brenier1991polar,mccann1995existence}, see also Theorem 1.28  on page 24 of \citep{philippis2013regularity} for details.
\end{proof}

\subsection{Proof of Lemma \ref{lem2}}
\begin{proof}
	Our proof follows \cite{keziou2003dual}.
	Since $f(t)$ is convex, then $\forall t \in \mathbb{R}$, we have  $f(t)=f^{**}(t)$, where $$f^{**}(t)=\sup_{s\in\mathbb{R}} \{st - f^*(s)\}$$ is the Fenchel conjugate of $f^*$. By  Fermat's rule, the maximizer  $s^*$  satisfies $$t\in \partial 
f^*(s^*),$$ i.e., $$s^*\in \partial f(t)$$
	Plugging  the above display with $t = \frac{\mathrm{d} \mu_Z}{\mathrm{d} \gamma}(x)$ into the definition of $f$-divergence,  we derive (\ref{fdual}).
\end{proof}
\subsection{Proof of Theorem \ref{glmin}}
\begin{proof}
	Without loss of generality, we assume $d=1$. For $R^*$ satisfying (\ref{cida2}) and  any $R\in \mathcal{R}$, we have $R= \rho_{(R,R^*)} R^* + \varepsilon_R$, where $\rho_{(R,R^*)}$ is the correlation coefficient between $R$ and $R^*$,  $\varepsilon_R= R- \rho_{(R,R^*)} R^*$. It is easy to see that $\varepsilon_R \indep R^*$ and thus $Y\indep \varepsilon_R $. As $(\rho_{(R,R^*)} R^*, Y)$ is independent of $( \varepsilon_R, 0)$, then by Theorem 3 of \cite{szekely2009dCov}
	\begin{align*}
	\mathcal{V}[R, \rvy] =& \mathcal{V}[\rho_{(R,R^*)} R^* + \varepsilon_R, \rvy]
	\le \mathcal{V}[\rho_{(R,R^*)} R^* , \rvy] + \mathcal{V}(\varepsilon_R, 0)\\
	= & \mathcal{V}[\rho_{(R,R^*)} R^* , \rvy]
	=  |\rho_{(R,R^*)} | \mathcal{V}[R^* , \rvy] \\
	\le &  \mathcal{V}[R^* , \rvy].
	\end{align*}
	As $R(\rvx)\sim \mathcal{N}(0,1)$ and $R^*(\rvx)\sim \mathcal{N}(0,1)$, then
	$\mathbb{D}_f(\mu_{R(\rvx)} \Vert \gamma_{d}) = \mathbb{D}_f(\mu_{R^*(\rvx)} \Vert \gamma_{d})=0$, and
	\begin{align*}
	\mathcal{L}(R)- \mathcal{L}(R^*) = 
\mathcal{V}[R^* , \rvy] -   
\mathcal{V}[R , \rvy] \ge 0.
	\end{align*}
	The proof is completed.
\end{proof}
\subsection{Proof of Theorem \ref{esterr}}

{\color{black}Recall that  $B_2 = \max\{|f^{\prime}(c_1)|,|f^{\prime}(c_2)|\}$,  $B_3 = \max_{|s|\leq  2 B_2} |f^*(s)|$.
We set the network parameters of the representer  $R_{\vtheta}$  and the discriminator  $D_{\vphi}$ as follows.
\begin{enumerate}
	\item[(N1)]
	Representer network $\mathbf{R}\equiv \mathbf{R}_{\mathcal{H}, \mathcal{W}, \mathcal{S}}$ parameters:
	depth $\mathcal{H} = \mathcal{O}( \log n)$ width
$\mathcal{W} = \mathcal{O}(n^{\frac{p}{2(2+p)}}/\log n),$  size $\mathcal{S} =\mathcal{O}(dn^{\frac{p}{2+p}}/\log^4 (npd)),$
and $\|R\|_{L^{\infty}} \leq  \mathcal{B} = 2\|R^*\|_{L^{\infty}}, \forall R \in \mathbf{R}.$
	\item[(N2)] Discriminator network $\mathbf{D}\equiv \mathbf{D}_{\tilde{\mathcal{H}}, \tilde{\mathcal{W}}, \tilde{\mathcal{S}}}$ parameters:
depth $\tilde{\mathcal{H}} = \mathcal{O}(\log n),$ width $\tilde{\mathcal{W}} = \mathcal{O}(n^{\frac{d}{2(2+d)}}/\log n),$ size $\tilde{\mathcal{S}} =\mathcal{O}(n^{\frac{d}{2+d}}/\log^4 (npd)),$
and
$\|D\|_{L^{\infty}} \leq   2B_2, \forall D \in \mathbf{D}.$
\end{enumerate}
}
	
Before getting into the details of the proof of Theorem \ref{esterr}, we first give an outline of the basic structure of the proof.

Without loss of generality, we assume that $\lambda = 1$ and $m = 1$, i.e. $\rvy\in \mathbb{R}$.
%
For any	$\bar{R}\in \mathbf{R}_{\mathcal{H}, \mathcal{W}, \mathcal{S}}$, we have,
\begin{align}
	\mathcal{L}(\widehat{R}_{\vtheta})-\mathcal{L}(R^{*})
	& =\mathcal{L}(\widehat{R}_{\vtheta}) - \widehat{\mathcal{L}}(\widehat{R}_{\vtheta}) +    \widehat{\mathcal{L}}(\widehat{R}_{\vtheta})-  \widehat{\mathcal{L}}(\bar{R}) +   \widehat{\mathcal{L}}(\bar{R}) - \mathcal{L}(\bar{R}) +\mathcal{L}(\bar{R})  - \mathcal{L}(R^*)\nonumber\\
	&\leq  2 \sup_{R \in \mathbf{R}_{\mathcal{H}, \mathcal{W}, \mathcal{S}}} |\mathcal{L}(R) - \widehat{\mathcal{L}}(R) |+ \inf_{\bar{R} \in \mathbf{R}_{\mathcal{H}, \mathcal{W}, \mathcal{S}}} |\mathcal{L}(\bar{R})  - \mathcal{L}(R^*)| , \label{A7}
	\end{align}
	where we use the definition of    $\widehat{R}_{\vtheta}$ in (\ref{dnpe}) and the feasibility of $\bar{R}$.
	Next we bound the two error terms in (\ref{A7}),
\begin{itemize}
\item the \textbf{approximation error}: $\inf_{\bar{R} \in \mathbf{R}_{\mathcal{H}, \mathcal{W}, \mathcal{S}, }} |\mathcal{L}(\bar{R})  - \mathcal{L}(R^*)|;$
\item the \textbf{statistical error}:  $\sup_{R \in \mathbf{R}_{\mathcal{H}, \mathcal{W}, \mathcal{S}}} |\mathcal{L}(R) - \widehat{\mathcal{L}}(R) |.$
\end{itemize}
Then Theorem \ref{esterr} follows after bounding  these two error terms.

\bigskip\noindent
{\textbf{A. The approximation error}}
\begin{lemma}\label{2.4}
Suppose that (A1)-(A3) hold and the network parameters satisfy (N1) and (N2). Then,
		\begin{equation}\label{appf}
		\inf_{\bar{R} \in \mathbf{R}_{\mathcal{D}, \mathcal{W}, \mathcal{S}, \mathcal{B}}} |\mathcal{L}(\bar{R})  - \mathcal{L}(R^*)| \leq 320C_1 L_1B_1\sqrt{pd} n^{-1/(p+2)}+o(1).
		\end{equation}
as $n\rightarrow \infty.$
	\end{lemma}
	\begin{proof}
		By (\ref{cida2}) and (\ref{fdual})   and the definition of $\mathcal{L}$, we have
		\begin{equation} \label{apperr}
		\inf_{\bar{R} \in \mathbf{R}_{\mathcal{D}, \mathcal{W}, \mathcal{S}, \mathcal{B}}}  |\mathcal{L}(\bar{R})  - \mathcal{L}(R^*)| \leq |\mathbb{D}_{f}(\mu_{\bar{R}_{\bar{\vtheta}}(\rvx)} \Vert \gamma_{d})| + |   \mathcal{V}[R^*(\rvx), \rvy]- \mathcal{V}[\bar{R}_{\bar{\vtheta}}(\rvx), \rvy]|,
		\end{equation}
		where $\bar{R}_{\bar{\vtheta}}\in \mathbf{R}_{\mathcal{D}, \mathcal{W}, \mathcal{S}, \mathcal{B}}$ is specified  in Lemma \ref{lemapp1} below.
		We finish the proof by (\ref{app2}) in Lemma \ref{2.5} and (\ref{app3}) in Lemma \ref{2.6}, which will be  proved  below.
	\end{proof}

\begin{lemma}\label{shen}
For any function $f:[-B,B]^p\rightarrow \mathbb{R}$ with Lipschitz constant $L$
there exist a ReLU network $\bar{f}$ with depth $\mathcal{O}(12\mathcal{H}+C_{1,p}) $ and width $\mathcal{O}(C_{2,p} \mathcal{W})$
such that $$\|f-\bar{f}\|_{L^{\infty}} \leq 19L\sqrt{p} B (\mathcal{H}\mathcal{W})^{-2/p},$$
where $C_{1,p} = 14+2p$, $C_{2,p} = 3^{p+3}.$
\begin{proof}
This Lemma fellows directly from  Theorem 1.1 of \cite{shen2019deep}
\end{proof}
\end{lemma}

\begin{lemma}\label{lemapp1}
Suppose that (A1) and (A3) hold and the network parameters satisfy (N1). Then,
		There exist a $\bar{R}_{\bar{\vtheta}} \in  \mathbf{R}_{\mathcal{H}, \mathcal{W}, \mathcal{S}}$ with the network parameters satisfying (N1)
		such that
		\begin{equation}\label{app1}
		\|\bar{R}_{\bar{\vtheta}} - {R}^*\|_{L^2(\mu_{\rvx})} \leq  19 L_1B_1\sqrt{pd} n^{-\frac{1}{p+2}}.
		\end{equation}
	\end{lemma}
	\begin{proof}
		Let ${R}^*_{i}(x)$ be the $i$-th entry  of  ${R}^*(x):\mathbb{R}^d \rightarrow \mathbb{R}^{d}$.
		By the assumption on ${R}^*$, it is easy to check that  ${R}^*_{i}(x)$ is  Lipschitz continuous on $[-B_1,B_1]^d$ with  the Lipschitz constant $L_1$.  By Lemma   \ref{shen},
there exists a ReLU network $\bar{R}_{\bar{\vtheta}_i}$ with depth $\mathcal{O}(\mathcal{H})$ and width $\mathcal{O}(\mathcal{W})$
 such that
		$$ \|{R}^*_{i}-\bar{R}_{\bar{\vtheta}_i}\|_{L^{\infty}} \leq 19 L_1B_1\sqrt{p} (\mathcal{H}\mathcal{W})^{-2/p}.$$
Then 		
 \begin{align*}
 &\|{R}^*_{i}-\bar{R}_{\bar{\vtheta}_i}\|_{L^{2}(\mu_{\rvx})}  = [\int ({R}^*_{i}(x)-\bar{R}_{\bar{\vtheta}_i}(x))^2 f_X(x)\mathrm{d}x]^{1/2}\\
 &\leq  \|{R}^*_{i}-\bar{R}_{\bar{\vtheta}_i}\|_{L^{\infty}}\int f_X(x)\mathrm{d}x\\
 &\leq  19 L_1B_1\sqrt{p} (\mathcal{H}\mathcal{W})^{-2/p}.
 \end{align*}
		Define  $\bar{R}_{\bar{\vtheta}} = [\bar{R}_{\bar{\vtheta}_1},\ldots,\bar{R}_{\bar{\vtheta}_{d}}] \in  \mathbf{R}_{\mathcal{H}, \mathcal{W}, \mathcal{S}}$.
		The above three display implies $$\|\bar{R}_{\bar{\vtheta}} - \tilde{R}^*\|_{L^2(\mu_{\rvx})} \leq  19 L_1B_1\sqrt{pd} (\mathcal{H}\mathcal{W})^{-2/p}\leq 19 L_1B_1\sqrt{pd} n^{-1/(p+2)},$$
where in the last inequality we use the the choice of $\mathcal{H}$ and $\mathcal{W}$ in (N1).
	\end{proof}
	\begin{lemma}\label{2.5}
Suppose that (A1) and (A3) hold and the network parameters satisfy (N1). Then,
		\begin{equation}\label{app2}
		| \mathcal{V}[R^*(\rvx), \rvy]- \mathcal{V}[\bar{R}_{\bar{\vtheta}}(\rvx), \rvy]|\leq 320C_1 L_1B_1\sqrt{pd} n^{-1/(p+2)}.
		\end{equation}
	\end{lemma}
	\begin{proof}
		Recall that \cite{szekely2007measuring}
		\begin{align*}
		\mathcal{V}[\rvz, \rvy]=& \mathbb{E}\left[\|\rvz_{1}-\rvz_{2}\||\rvy_{1}-\rvy_{2}|\right]-2 \mathbb{E}\left[\|\rvz_{1}-\rvz_{2}\||\rvy_{1}-\rvy_{3}|\right] \\
		&+\mathbb{E}\left[\|\rvz_{1}-\rvz_{2}\|\right] \mathbb{E}\left[|\rvy_{1}-\rvy_{2}|\right],
		\end{align*}
		where $(\rvz_i,\rvy_i), i = 1,2,3$ are i.i.d. copies of $(\rvz,\rvy)$.
		We have
		\begin{align*}
		&| \mathcal{V}[R^*(\rvx), \rvy]- \mathcal{V}[\bar{R}_{\bar{\vtheta}}(\rvx), \rvy]|\\
		&\leq  |\mathbb{E}\left[(\|R^*(\rvx_1)-R^*(\rvx_2)\|-\|\bar{R}_{\bar{\vtheta}}(\rvx_1)-\bar{R}_{\bar{\vtheta}}(\rvx_2)\|)|\rvy_{1}-\rvy_{2}|\right]|\\
		&+2 |\mathbb{E}\left[(\|R^*(\rvx_1)-R^*(\rvx_2)\|-\|\bar{R}_{\bar{\vtheta}}(\rvx_1)-\bar{R}_{\bar{\vtheta}}(\rvx_2)\|)|\rvy_{1}-\rvy_{3}|\right]|\\
		&+|\mathbb{E}\left[\|R^*(\rvx_1)-R^*(\rvx_2)\|-\|\bar{R}_{\bar{\vtheta}}(\rvx_1)-\bar{R}_{\bar{\vtheta}}(\rvx_2)\right] \mathbb{E}\left[\left\|\rvy_{1}-\rvy_{2}\right\|\right]|\\
		& \leq 8C_1\mathbb{E}\left[|\|R^*(\rvx_1)-R^*(\rvx_2)\|-\|\bar{R}_{\bar{\vtheta}}(\rvx_1)-\bar{R}_{\bar{\vtheta}}(\rvx_2)\||\right]\\
		&\leq  16C_1\mathbb{E}\left[|\|R^*(\rvx)-\bar{R}_{\bar{\vtheta}}(\rvx)\|\right]\\
       & \leq 320C_1 L_1B_1\sqrt{pd} n^{-1/(p+2)}
		\end{align*}
		where in the first and third inequalities we use the triangle inequality, and second one follows from the boundedness of $\rvy$,
the last inequality is due to  (\ref{app1}).
	\end{proof}
	\begin{lemma}\label{2.6}
Suppose that (A1) and (A2) hold and the network parameters satisfy (N1). Then,
		\begin{equation}\label{app3}
		|\mathbb{D}_{f}(\mu_{\bar{R}_{\bar{\vtheta}}(\rvx)} \Vert \gamma_{d})| \rightarrow 0,
		\end{equation}
as $n\rightarrow\infty.$
	\end{lemma}
	\begin{proof}
		By Lemma \ref{lemapp1} $\bar{R}_{\bar{\vtheta}}$ can approximate $R^*$ arbitrarily well as $n\rightarrow \infty$,
		the desired result follows from the fact that    $ \mathbb{D}_{f}(\mu_{R^*(\rvx)} \Vert \gamma_{d}) = 0$ and the continuity of  $ \mathbb{D}_{f}(\mu_{R(\rvx)} \Vert \gamma_{d})$ on $R$. We present  the sketch of the proof  and omit the details here.
		Let $r^*(z) = \frac{\mathrm{d}\mu_{R^*(\rvx)}}{\mathrm{d}\gamma_{d}}(z)$ and $\bar{r}(z) = \frac{\mathrm{d}\mu_{\bar{R}_{\bar{\vtheta}}(\rvx)}}{\mathrm{d}\gamma_{d}}(z).$
		By definition we have
		\begin{eqnarray*}
	\mathbb{D}_{f}(\mu_{R^*(\rvx)} \Vert \gamma_{d}) &=& \mathbb{E}_{W\sim \gamma_{d}}[f(r^*(W))]\\
		\end{eqnarray*}
We can represent $\mathbb{D}_{f}(\mu_{\bar{R}_{\bar{\vtheta}}} \Vert \gamma_{d})$ similarly.
Therefore,
		\begin{align*}
		&|\mathbb{D}_{f}(\mu_{\bar{R}_{\bar{\vtheta}}(\rvx)} \Vert \gamma_{d})|  = |\mathbb{D}_{f}(\mu_{\bar{R}_{\bar{\vtheta}}(\rvx)} \Vert \gamma_{d})-\mathbb{D}_{f}(\mu_{R^*(\rvx)} \Vert \gamma_{d})|\\
		&\leq \mathbb{E}_{W\sim \gamma_{d}}[|f(r^*(W))-f(\bar{r}(W))|]\\
		& \leq \int |f^{\prime}(\tilde{r}(z))||r^*(z)-\bar{r}(z)| \mathrm{d}\gamma_{d}(z)\\
		&\leq B_2 \int |r^*(z)-\bar{r}(z)| \mathrm{d}\gamma_{d}(z),
		\end{align*}
where the second  inequality we use mean value theorem and  boundness  assumption on $f^{\prime}(\tilde{r})$ in  (A2).
Then last inequality goes to zero due to continuity and the fact  $\bar{R}_{\bar{\vtheta}}$ converge to  $R^*$ as $n\rightarrow \infty$  by Lemma \ref{lemapp1}. 
	\end{proof}
	
\bigskip\noindent
{\textbf{B. The statistical  error}}
\begin{lemma}\label{lemasta}
Suppose that (A1)-(A2) hold and the network parameters satisfy (N1) and (N2). Then,
		\begin{equation}\label{sta1}
		\sup_{R \in \mathbf{R}_{{\mathcal{H}}, {\mathcal{W}}, {\mathcal{S}}}} |\mathcal{L}(R) - \widehat{\mathcal{L}}(R) |
		\leq C_{13}(2 B_2+ B_3)n^{-\frac{1}{2+d}} +  19(1+B_{3})   L_2\sqrt{d}\log n n^{-\frac{1}{d+2}}  + 4C_6 C_7C_{10}\mathcal{B}n^{-\frac{1}{p+2}}
		\end{equation}
	\end{lemma}
	
	\begin{proof}
		By the definition and the triangle inequality we have
		\begin{eqnarray*}
\mathbb{E}[\sup_{R \in \mathbf{R}_{{\mathcal{H}}, {\mathcal{W}}, {\mathcal{S}}}} |\mathcal{L}(R) - \widehat{\mathcal{L}}(R) |]
		&\leq & \mathbb{E} [\sup_{R \in \mathbf{R}_{{\mathcal{H}}, {\mathcal{W}}, {\mathcal{S}}}} |\widehat{\mathcal{V}}_{n}[R(\rvx), \rvy] - \mathcal{V}[(R(\rvx), \rvy)|]\\
		& &+  \mathbb{E}[\sup_{R \in \mathbf{R}_{{\mathcal{H}}, {\mathcal{W}}, {\mathcal{S}}}} |\widehat{\mathbb{D}}_{f}(\mu_{R(\rvx)}|| \gamma_{d}) - \mathbb{D}_{f}(\mu_{R(\rvx)}|| \gamma_{d})|].
		\end{eqnarray*}
We finish the proof based on (\ref{sta2}) in Lemma \ref{2.8} and (\ref{sta3}) in Lemma \ref{B7}, which will be  proved  below.
	\end{proof}
	
	\begin{lemma}\label{2.8}
Suppose that (A1)-(A2) hold and the network parameters satisfy (N1) and (N2). Then,
		\begin{equation}\label{sta2}
		\mathbb{E}[\sup_{R \in \mathbf{R}_{{\mathcal{H}}, {\mathcal{W}}, {\mathcal{S}}}} |\widehat{\mathcal{V}}_{n}[R(\rvx), \rvy] - \mathcal{V}[R(\rvx), \rvy]|] \leq
		4C_6 C_7C_{10}\mathcal{B}n^{-\frac{1}{p+2}}.
		\end{equation}
	\end{lemma}
	\begin{proof}
		We first fix some notation for simplicity.
		Denote $O = (\rvx,\rvy)\in \mathbb{R}^{p} \times \mathbb{R}^{1} $ and $O_i = (\rvx_i,\rvy_i), i=1,...n$ are i.i.d copy of $O$, and  denote $\mu_{\rvx,\rvy}$ and  $\mathbb{P}^{\otimes n}$ as $ \mathbb{P}$ and  $\mathbb{P}^{n}$, respectively.
		$\forall R \in {R \in \mathbf{R}_{\tilde{\mathcal{H}}, \tilde{\mathcal{W}}, \tilde{\mathcal{S}}}}$, let  $\tilde{O} = (R(\rvx),\rvy) $ and $\tilde{O}_i = (R(\rvx_i),\rvy_i), i=1,...n$ are i.i.d copy of $\tilde{O}$.
		Define centered kernel $\bar{h}_{R}:(\mathbb{R}^{p} \times \mathbb{R}^{1})^{\otimes 4} \rightarrow \mathbb{R}$ as
		\begin{equation}\label{newkernel}
		\begin{array}{l}
		\bar{h}_{R}(\tilde{O}_1,\tilde{O}_2, \tilde{O}_3,\tilde{O}_4) =\frac{1}{4} \sum_{1 \leq i, j \leq 4, \atop i \neq j}\|R(\rvx_{i})-R(\rvx_{j})\| |\rvy_{i}-\rvy_{j}|\\
		-\frac{1}{4} \sum_{i=1}^{4}\left(\sum_{1 \leq j \leq 4, \atop j \neq i}\left\|R(\rvx_{i})-R(\rvx_{j})\right\| \sum_{1 \leq j \leq 4,\atop i \neq j}|\rvy_{i}-\rvy_{j}|\right) \\
		\quad+\frac{1}{24} \sum_{1 \leq i, j \leq 4, \atop i \neq j}\left\|R(\rvx_{i})-R(\rvx_{j})\right\| \sum_{1 \leq i, j \leq 4, \atop i \neq j}|\rvy_{i}-\rvy_{j}| -\mathcal{V}[R(\rvx), \rvy]
		\end{array}.
		\end{equation}
		Then,  the  centered $U$-statistics $\widehat{\mathcal{V}}_{n}[R(\rvx), \rvy] - \mathcal{V}[R(\rvx), \rvy]$ can be represented as
		$$
		\mathbb{U}_{n}(\bar{h}_R) = \frac{1}{C_{n}^{4}} \sum_{1 \leq i_{1}<i_{2}<i_{3}<i_{4} \leq n} \bar{h}_{R}(\tilde{O}_{i_1},\tilde{O}_{i_2}, \tilde{O}_{i_3},\tilde{O}_{i_4}).$$
Our goal is to bound the supremum  of the centered $U$-process $\mathbb{U}_{n}(\bar{h}_{R})$ with the nondegenerate kernel  $\bar{h}_{R}$.
		By the symmetrization  randomization Theorem 3.5.3 in \cite{de2012decoupling}, we have
		\begin{equation}\label{rpm}
		\mathbb{E}[\sup_{R \in \mathbf{R}_{{\mathcal{H}}, {\mathcal{W}}, {\mathcal{S}}}}|\mathbb{U}_{n}(\bar{h}_{R})|]
		\leq C_5 \mathbb{E}[\sup_{R \in \mathbf{R}_{{\mathcal{H}}, {\mathcal{W}}, {\mathcal{S}}}}|  \frac{1}{C_{n}^{4}} \sum_{1 \leq i_{1}<i_{2}<i_{3}<i_{4} \leq n} \epsilon_{i_1}\bar{h}_{R}(\tilde{O}_{i_1},\tilde{O}_{i_2}, \tilde{O}_{i_3},\tilde{O}_{i_4})|],
		\end{equation}
		where, $\epsilon_{i_1}, i_1 = 1,...n$ are i.i.d  Rademacher variables that are also independent with $\tilde{O}_i, i = 1,\ldots, n.$
		We finish the proof by upper bounding the above Rademacher process with the metric entropy of  ${R \in \mathbf{R}_{\tilde{\mathcal{H}}, \tilde{\mathcal{W}}, \tilde{\mathcal{S}}}}$.
To this end we need the following lemma.

\begin{lemma}
\label{B7}
			If $\xi_{i}, i = 1,...m$ are $m$ finite linear combinations of Rademacher variables $\epsilon_j, j=1,..J$.
			Then
			\begin{equation}\label{rma}
			\mathbb{E}_{\epsilon_j,j=1,...J} \max _{1\leq i \leq m} |\xi_{i}| \leq C_6 (\log m)^{1/2} \max _{1\leq i \leq m}\left(\mathbb{E} \xi_{i}^{2}\right)^{1 / 2}.
			\end{equation}
		\end{lemma}
		\begin{proof}
			This result follows directly from Corollary 3.2.6 and  inequality (4.3.1)  in \cite{de2012decoupling} with $\Phi(x) = \exp(x^2).$
		\end{proof}
By Lemma \ref{lemapp1}, we can assume the boundness of  $ {R \in \mathbf{R}_{{\mathcal{H}}, {\mathcal{W}}, {\mathcal{S}}}}$, i.e., we can assume $\|R\|_{L^{\infty}} \leq \mathcal{B} = 2 \|R^*\|_{L^{\infty}}$ as $n $ large enough.
		Then by the boundedness assumption on $\rvy$,  we have that the kernel $\bar{h}_{R}$ is also bounded, say
		\begin{equation}\label{bd}
		\|\bar{h}_{R}\|_{L^{\infty}}\leq  C_7\mathcal{B}.
		\end{equation}
		$\forall {R \in \mathbf{R}_{{\mathcal{H}}, {\mathcal{W}}, {\mathcal{S}}}},$
		define a random empirical measure  (depends on ${O}_i, i=1,\ldots, n$)
		$$
		e_{n,1}(R, \tilde{R})= \mathbb{E}_{\epsilon_{i_1}, i_1 = 1,...,n}|\frac{1}{C_{n}^{4}} \sum_{1\leq i_1<  i_2<i_3<i_4\leq n} \epsilon_{i_{1}}(\bar{h}_{R}-\bar{h}_{\tilde{R}})(\tilde{O}_{i_{1}}, \ldots, \tilde{O}_{i_{4}})|.$$
		Condition on ${O}_i, i=1,\ldots, n$, let $\mathfrak{C}(\mathbf{R}, e_{n,1}, \delta))$ be  the covering number of ${\mathbf{R}_{{\mathcal{H}}, {\mathcal{W}}, {\mathcal{S}}}}$ with respect  to the  empirical  distance  $e_{n,1}$ at scale of $\delta>0$. Denote $\mathbf{R}_{\delta}$ as  the  covering set of  $\mathbf{R}_{\mathcal{D}, \mathcal{W}, \mathcal{S}}$ with cardinality of $\mathfrak{C}(\mathbf{R}, e_{n,1}, \delta))$.
		Then,
		\begin{align*}
		&\mathbb{E}_{\epsilon_{i_1}}[\sup_{R \in \mathbf{R}_{{\mathcal{H}}, {\mathcal{W}}, {\mathcal{S}}}}|  \frac{1}{C_{n}^{4}} \sum_{1 \leq i_{1}<i_{2}<i_{3}<i_{4} \leq n} \epsilon_{i_1}\bar{h}_{R}(\tilde{O}_{i_1},\tilde{O}_{i_2}, \tilde{O}_{i_3},\tilde{O}_{i_4})|]\\
		&\leq \delta + \mathbb{E}_{\epsilon_{i_1}}[\sup_{R \in \mathbf{R}_{\delta}}|  \frac{1}{C_{n}^{4}} \sum_{1 \leq i_{1}<i_{2}<i_{3}<i_{4} \leq n} \epsilon_{i_1}\bar{h}_{R}(\tilde{O}_{i_1},\tilde{O}_{i_2}, \tilde{O}_{i_3},\tilde{O}_{i_4})|]\\
		&\leq \delta + C_6 \frac{1}{C_{n}^{4}} (\log \mathfrak{C}(\mathbf{R}, e_{n,1}, \delta))^{1/2} \max_{R \in \mathbf{R}_{\delta}} [\sum_{i_1 = 1}^n \sum_{i_2<i_3<i_4} (\bar{h}_{R}(\tilde{O}_{i_1},\tilde{O}_{i_2}, \tilde{O}_{i_3},\tilde{O}_{i_4}))^2]^{1/2}\\
		&\leq \delta + C_6 C_7\mathcal{B} (\log \mathfrak{C}(\mathbf{R}, e_{n,1}, \delta))^{1/2} \frac{1}{C_{n}^{4}}[\frac{n(n !)^{2}}{((n-3) !)^{2}}]^{1 / 2}\\
		&\leq  \delta + 2C_6 C_7\mathcal{B}(\log \mathfrak{C}(\mathbf{R}, e_{n,1}, \delta))^{1/2}/\sqrt{n}\\
		&\leq \delta + 2C_6 C_7\mathcal{B} (\mathrm{VC}_{\mathbf{R}} \log \frac{2e\mathcal{B}n}{\delta\mathrm{VC}_{\mathbf{R}}})^{1/2}/\sqrt{n}\\
		& \leq  \delta + C_6 C_7C_{10}\mathcal{B} ( \mathcal{H}\mathcal{S}\log \mathcal{S} \log \frac{\mathcal{B}n}{\delta \mathcal{D}\mathcal{S}\log \mathcal{S}})^{1/2}/\sqrt{n}.
		\end{align*}
		where the first inequality follows from the triangle inequality,  the second inequality uses (\ref{rma}),  the third and fourth inequalities follow after some algebra,
		and the fifth inequality holds since
		$\mathfrak{C}(\mathbf{R},e_{n,1},\delta)\leq \mathfrak{C}(\mathbf{R},e_{n,\infty},\delta)$ and
		the relationship between the metric entropy and the VC-dimension of the ReLU networks  $\mathbf{R}_{\mathcal{H}, \mathcal{W}, \mathcal{S}}$ \citep{anthony2009neural}, i.e.,
		$$\log \mathfrak{C}(\mathbf{R},e_{n,\infty},\delta)) \leq \mathrm{VC}_{\mathbf{R}} \log \frac{2e\mathcal{B}n}{\delta\mathrm{VC}_{\mathbf{R}}},$$
		and the last inequality  holds due to the  upper bound of VC-dimension for the ReLU network $\mathbf{R}_{\mathcal{D}, \mathcal{W}, \mathcal{S}}$ satisfying  $$C_8 \mathcal{H}\mathcal{S}\log \mathcal{S}\leq \mathrm{VC}_{\mathbf{R}} \leq C_9 \mathcal{H}\mathcal{S}\log \mathcal{S},$$ see \cite{bartlett2019}.
		Then (\ref{sta2}) holds by the selection of the network parameters and set $\delta = \frac{1}{n}$ and some algebra.
	\end{proof}

	\begin{lemma}\label{2.10} 
Suppose that (A1)-(A3) hold and the network parameters satisfy (N1) and (N2). Then,
		{\small
			\begin{equation}\label{sta3}
			\mathbb{E}[\sup_{R \in \mathbf{R}_{\mathcal{H}, \mathcal{W}, \mathcal{S}}} |\widehat{\mathbb{D}}_{f}(\mu_{R(\rvx)}|| \gamma_{d}) - \mathbb{D}_{f}(\mu_{R(\rvx)}|| \gamma_{d})|] \leq C_{14}(L_2\sqrt{d}+B_2+B_3)(n^{-\frac{2}{2+p}}+\log n n^{-\frac{2}{2+d}})
			\end{equation}}
	\end{lemma}
	\begin{proof}
		For every $R \in\mathbf{R}_{\mathcal{H}, \mathcal{W}, \mathcal{S}}$, let $r(z) = \frac{\mathrm{d}\mu_{R(\rvx)}}{\mathrm{d}\gamma_{d}}(z)$, $g_{R}(z) = f^{\prime}(r(z))$.
		By assumption $g_{R}(z):\mathbb{R}^{d}\rightarrow\mathbb{R}$ is  Lipschitz continuous with  the Lipschitz constant $L_2$ and
		$\|g_{R}\|_{L^{\infty}} \leq B_2$. By tail probability of Gaussian, we assume   without loss of generality that   $\mathrm{supp}(g_{R}) \subseteq [-\log n, \log n]^{d}$.
Then, by Lemma \ref{shen}
		there exists a $\bar{D}_{\bar{\vphi}} \in \mathbf{D}_{\tilde{\mathcal{H}}, \tilde{\mathcal{W}}, \tilde{\mathcal{S}}}$
		with the network parameters satisfying (N2)
		such that for $\rvz\sim \gamma_{d}$ and $\rvz\sim \mu_{R(\rvx)},$
		\begin{equation}\label{app4}
		\mathbb{E}_{\rvz}[|\bar{D}_{\bar{\vphi}}(\rvz) - g_{R}(\rvz)|] \leq  19 L_2\sqrt{d}\log n n^{-\frac{1}{d+2}}.
		\end{equation}
By the above display, we  can further assume that the element in $\mathbf{D}_{\tilde{\mathcal{H}}, \tilde{\mathcal{W}}, \tilde{\mathcal{S}}}$ is bounded by
$2B_2$ as $n$ large enough.
	For any	$ g:\mathbb{R}^{d}\rightarrow \mathbb{R}$, define $$\mathcal{E}(g) = \mathbb{E}_{\rvx\sim \mu_{\rvx}} [g(R(\rvx))]-\mathbb{E}_{W\sim \gamma_{d}} [f^*(g(W))],$$
		$$\widehat{\mathcal{E}}(g) = \widehat{\mathcal{E}}(g,R)=\frac{1}{n} \sum_{i=1}^n [g(R(\rvx_i))- f^*(g(W_i))].$$
		By (\ref{fdual}) we have
		\begin{equation}\label{sup}
		\mathcal{E}(g_R) = \mathbb{D}_{f}(\mu_{R(\rvx)}|| \gamma_{d}) = \sup_{\mathrm{measureable} \ \  D:\mathbb{R}^{d}\rightarrow \mathbb{R}} \mathcal{E}(D).
		\end{equation}
Then,
		\begin{align*}
		&| \mathbb{D}_{f}(\mu_{R(\rvx)}|| \gamma_{d})-\widehat{\mathbb{D}}_{f}(\mu_{R(\rvx)}|| \gamma_{d})|\\
		&= |  \mathcal{E}(g_R)  -  \max_{D_{\vphi} \in \mathbf{D}_{\tilde{\mathcal{H}}, \tilde{\mathcal{W}}, \tilde{\mathcal{S}}, }}\widehat{\mathcal{E}}(D_{\vphi})|\\
		&\leq |\mathcal{E}(g_R)-\sup_{D_{\vphi} \in \mathbf{D}_{\tilde{\mathcal{H}}, \tilde{\mathcal{W}}, \tilde{\mathcal{S}}, }}\mathcal{E}(D_{\vphi})| + |\sup_{D_{\vphi} \in \mathbf{D}_{\tilde{\mathcal{H}}, \tilde{\mathcal{W}}, \tilde{\mathcal{S}}, }}\mathcal{E}(D_{\vphi})-\max_{D_{\vphi} \in \mathbf{D}_{\tilde{\mathcal{H}}, \tilde{\mathcal{W}}, \tilde{\mathcal{S}}, }}\widehat{\mathcal{E}}(D_{\vphi})|\\
		&\leq |\mathcal{E}(g_R) - \mathcal{E}(\bar{D}_{\bar{\vphi}})|+\sup_{D_{\vphi} \in \mathbf{D}_{\tilde{\mathcal{H}}, \tilde{\mathcal{W}}, \tilde{\mathcal{S}}, }}|\mathcal{E}(D_{\vphi})-\widehat{\mathcal{E}}(D_{\vphi})|\\
		&\leq \mathbb{E}_{\rvz\sim \mu_{R(\rvx)}} [|g_R-\bar{D}_{\bar{\vphi}}|(\rvz)] + \mathbb{E}_{W\sim \gamma_{d}} [|f^*(g_{R})-f^*(\bar{D}_{\bar{\vphi}})|(W)] + \sup_{D_{\vphi} \in \mathbf{D}_{\tilde{\mathcal{D}}, \tilde{\mathcal{W}}, \tilde{\mathcal{S}}, }}|\mathcal{E}(D_{\vphi})-\widehat{\mathcal{E}}(D_{\vphi})|\\
		&\leq 19(1+B_{3})   L_2\sqrt{d}\log n n^{-\frac{1}{d+2}} + \sup_{D_{\vphi} \in \mathbf{D}_{\tilde{\mathcal{H}}, \tilde{\mathcal{W}}, \tilde{\mathcal{S}}, \tilde{\mathcal{B}}}}|\mathcal{E}(D_{\vphi})-\widehat{\mathcal{E}}(D_{\vphi})|,
		\end{align*}
where we use the triangle inequality in the first inequality follows from the triangle inequality,  the second inequality follows from  $\mathcal{E}(g_R)\geq \sup_{D_{\vphi} \in \mathbf{D}_{\tilde{\mathcal{H}}, \tilde{\mathcal{W}}, \tilde{\mathcal{S}}, \tilde{\mathcal{B}}}}\mathcal{E}(D_{\vphi})$ due to  (\ref{sup})  and the triangle inequality, the third inequality follows from the triangle inequality, and the last inequality follows from (\ref{app4}) and the mean value theorem.

We finish the proof by bounding the second term in probability
in the last line above, i.e.,
$\sup_{D_{\vphi} \in \mathbf{D}_{\tilde{\mathcal{H}}, \tilde{\mathcal{W}}, \tilde{\mathcal{S}}, \tilde{\mathcal{B}}}}|\mathcal{E}(D_{\vphi})-\widehat{\mathcal{E}}(D_{\vphi})|.$
This can be done by bounding the
 empirical process
		$$\mathbb{U}(D,R)=\mathbb{E}[\sup_{R \in \mathbf{R}_{\mathcal{H}, \mathcal{W}, \mathcal{S}, \mathcal{B}}, D \in \mathcal{D}_{\tilde{\mathcal{H}}, \tilde{\mathcal{W}}, \tilde{\mathcal{S}}, \tilde{\mathcal{B}}}}|\mathcal{E}(D)-\widehat{\mathcal{E}}(D)|].$$
		Let $S = (\rvx,\rvz) \sim \mu_\rvx\otimes\gamma_{d}$ and $S_i, i=1,\ldots, n$ be $n$ i.i.d copy of $S$. Denote $$b(D,R;S) = D(R(\rvx)) - f^*(D(\rvz)).$$
		Then
		$$\mathcal{E}(D,R) = \mathbb{E}_{S} [b(D,R;S)]$$
		and
		$$ \widehat{\mathcal{E}}(D,R) = \frac{1}{n}\sum_{i=1}^n b(D,R;S_i).$$
		Let
		$$\mathcal{G}(\mathbf{D}\times\mathbf{R}) = \frac{1}{n} \mathbb{E}_{\{S_i, \epsilon_i\}_{i}^n}\left[\sup_{R \in \mathbf{R}_{\mathcal{H}, \mathcal{W}, \mathcal{S}, \mathcal{B}}, D \in \mathbf{D}_{\tilde{\mathcal{H}}, \tilde{\mathcal{W}}, \tilde{\mathcal{S}}, \tilde{\mathcal{B}}}}|\sum_{i=1}^n\epsilon_i b(D,R;S_i)|\right]$$
be the 	
		Rademacher  complexity of $ \mathbf{D}_{\tilde{\mathcal{H}}, \tilde{\mathcal{W}}, \tilde{\mathcal{S}}}\times \mathbf{R}_{\mathcal{H}, \mathcal{W}, \mathcal{S}} $ \citep{bartlett2002rademacher}.
		Let  $\mathfrak{C}(\mathbf{D}\times \mathbf{R}, e_{n,1}, \delta))$ be  the covering number of $ \mathbf{D}_{\tilde{\mathcal{H}}, \tilde{\mathcal{W}}, \tilde{\mathcal{S}}}\times \mathbf{R}_{\mathcal{H}, \mathcal{W}, \mathcal{S}}$ with respect to the  empirical  distance (depends on $S_i$)  $$d_{n,1}((D,R),(\tilde{D},\tilde{R})) = \frac{1}{n} \mathbb{E}_{\epsilon_i}[\sum_{i =1}^n |\epsilon_i (b(D,R;S_i) - b(\widetilde{D},\tilde{R};S_i))|]$$ at scale of $\delta>0$. Let $\mathbf{D}_{\delta}\times\mathbf{R}_{\delta}$ be  such a converging set of
		$ \mathbf{D}_{\tilde{\mathcal{H}}, \tilde{\mathcal{W}}, \tilde{\mathcal{S}}}\times \mathbf{R}_{\mathcal{H}, \mathcal{W}, \mathcal{S}}$.
		Then,
		\begin{align*}
		&\mathbb{U}(D,R) = 2\mathcal{G}(\mathbf{D}\times\mathbf{R})\\
		&= 2\mathbb{E}_{S_1,\ldots,S_n}[\mathbb{E}_{\epsilon_i,i=1,\ldots,n}[\mathcal{G}(\mathbf{R}\times\mathbf{D})|(S_1,...,S_n)]]\\
		&\leq  2\delta +  \frac{2}{n}\mathbb{E}_{S_1,\ldots,S_n}[\mathbb{E}_{\epsilon_i,i=1,\ldots,n}[\sup_{(D,R)\in \mathbf{D}_{\delta}\times\mathbf{R}_{\delta}} |\sum_{i=1}^n\epsilon_i b(D,R;S_i)| | (S_1,\ldots,S_n)]\\
		&\leq  2\delta+C_{12}\frac{1}{n}\mathbb{E}_{S_1,\ldots,S_n}[(\log \mathfrak{C}(\mathbf{D}\times\mathbf{R}, e_{n,1}, \delta))^{1/2} \max_{(D,R) \in \mathbf{D}_{\delta}\times \mathbf{R}_{\delta}} [\sum_{i = 1}^n  b^2(D,R;S_i)]^{1/2}]\\
		&\leq 2\delta + C_{12}\frac{1}{n} \mathbb{E}_{S_1,\ldots,S_n}[(\log \mathfrak{C}(\mathbf{D}\times\mathbf{R}, e_{n,1}, \delta))^{1/2}\sqrt{n}(2B_2+ B_3)]\\
		&\leq 2\delta + C_{12}\frac{1}{\sqrt{n}}(2 B_2+ B_3)(\log \mathfrak{C}(\mathbf{D}, e_{n,1}, \delta)+\log \mathfrak{C}(\mathbf{R}, d_{n,1}, \delta))^{1/2}\\
		&\leq 2\delta +  C_{13}\frac{2 B_2+ B_3}{\sqrt{n}}( \mathcal{H}\mathcal{S}\log \mathcal{S} \log \frac{\mathcal{B}n}{\delta \mathcal{H}\mathcal{S}\log \mathcal{S}} + \tilde{\mathcal{H}}\tilde{\mathcal{S}}\log \tilde{\mathcal{S}} \log \frac{2B_2n}{\delta \tilde{\mathcal{H}}\tilde{\mathcal{S}}\log \tilde{\mathcal{S}}} )^{1/2}
		\end{align*}
		where the first equality follows from the standard symmetrization technique,   the first  inequality holds due to the iteration law of conditional expectation,
the second  inequality follows from the triangle inequality, and the third  inequality uses   \eqref{rma},  the fourth inequality uses the fact that
		$b(D,R;S)$  is bounded, i.e., $\|b(D,R;S)\|_{L^{\infty}} \leq 2 B_2+ B_3,$ and
		the fifth  inequality follows from some algebra, and the sixth inequality follows from   $\mathfrak{C}(\mathbf{R},e_{n,1},\delta)\leq \mathfrak{C}(\mathbf{R},e_{n,\infty},\delta)$ (similar result for $\mathbf{D}$) and
		$\log \mathfrak{C}(\mathbf{R},e_{n,\infty},\delta)) \leq \mathrm{VC}_{\mathbf{R}} \log \frac{2e\mathcal{B}n}{\delta\mathrm{VC}_{\mathbf{R}}},$
		and $\mathbf{R}_{\mathcal{H}, \mathcal{W}, \mathcal{S}, \mathcal{B}}$ satisfying  $C_8 \mathcal{H}\mathcal{S}\log \mathcal{S}\leq \mathrm{VC}_{\mathbf{R}} \leq C_9 \mathcal{H}\mathcal{S}\log \mathcal{S},$ see \cite{bartlett2019}.
		Then (\ref{sta3}) follows from the above display with  the selection of the network parameters of $ \mathbf{D}_{\tilde{\mathcal{H}}, \tilde{\mathcal{W}}, \tilde{\mathcal{S}}}, \mathbf{R}_{\mathcal{H}, \mathcal{W}, \mathcal{S}}$  and  with $\delta = \frac{1}{n}$.
	\end{proof}
Finally, Theorem \ref{esterr} is a direct consequence of (\ref{appf}) in Lemma \ref{2.4} and (\ref{sta1}) in Lemma \ref{lemasta}.
This completes the proof of Theorem \ref{esterr}. $\hfill\Box$

{
\singlespacing
\bibliography{ddr_jasa_bib}
}

\end{document}